\documentclass[accepted]{uai2026} 
                        
\usepackage[british]{babel}

\usepackage{natbib} 
\usepackage{mathtools} 
\usepackage{booktabs} 
\usepackage{tikz} 

\usepackage{csquotes}
\usepackage{tikz}
\usepackage{mathtools}
\usepackage{hyperref}       
\usepackage{cleveref}
\usepackage{amsmath}
\usepackage{amssymb}
\usepackage{amsthm}
\usepackage{thmtools}

\declaretheorem[name=Theorem]{theorem}
\declaretheorem[name=Lemma,sibling=theorem]{lemma}
\declaretheorem[name=Proposition,sibling=theorem]{proposition}

\declaretheorem[name=Definition,sibling=theorem]{definition}
\declaretheorem[name=Remark,sibling=theorem]{remark}

\usepackage{enumitem}
\usepackage{subcaption}

\newcommand{\bI}{\mathbb{I}}
\newcommand{\R}{{\mathbb R}}
\newcommand{\N}{{\mathbb N}}

\newcommand{\cX}{{\cal X}}

\newcommand{\cN}{\mathcal{N}}

\newcommand{\bE}{\mathbb{E}}

\newcommand{\cR}{\mathcal{R}}

\newcommand{\cG}{{\cal G}}
\newcommand{\cC}{{\cal C}}

\newcommand{\KL}[2]{D(#1\mid\mkern-1mu \mid #2)}
\newcommand{\Ex}[2][]{ \mathbb{E}_{#1}\left[#2\right] }
\newcommand{\chan}[0]{P_{Y\mid X}}

\newcommand{\PA}{\mathrm{PA}}
\newcommand{\pa}{\mathrm{pa}}

\newcommand\independent{\protect\mathpalette{\protect\independenT}{\perp}}
\def\independenT#1#2{\mathrel{\rlap{$#1#2$}\mkern2mu{#1#2}}}
\newcommand{\ind}{\independent}

\newtoggle{printcomments}
\toggletrue{printcomments} 
\togglefalse{printcomments} 

\newcommand\dominik[1]{%
  \iftoggle{printcomments}{%
    \textcolor{cyan}{Dominik: #1}%
  }{}%
}

\newcommand\will[1]{%
  \iftoggle{printcomments}{%
\textcolor{orange}{Will: #1} 
  }{}%
}

\definecolor{lightgray}{gray}{0.85}

\usetikzlibrary{arrows,shapes,plotmarks,positioning}
\usetikzlibrary{decorations.markings}

\tikzset{>=stealth'} 
\tikzset{node distance=2cm}
\tikzstyle{graphnode} = 
   [circle,draw=black,minimum size=22pt,text centered,text
     width=22pt,inner sep=0pt] 
\tikzstyle{var}   =[graphnode,fill=white]
\tikzstyle{vardashed}   =[graphnode,draw=gray,fill=white]
\tikzstyle{obs}   =[graphnode,fill=black,text=white]
\tikzstyle{obsgrey}   =[graphnode,draw=white,fill=lightgray,text=black]
\tikzstyle{par}    =[graphnode,draw=white,fill=red,text=black] 
 \tikzstyle{crucial} =[graphnode,draw=white,fill=yellow,text=black] 
\tikzstyle{fac}   =[rectangle,draw=black,fill=black!25,minimum size=5pt]
\tikzstyle{facprior} =[rectangle,draw=black,fill=black,text=white,minimum size=5pt]
\tikzstyle{edge}  =[draw=white,double=black,very thick,-]
\tikzstyle{blueedge}  =[draw=white,double=blue,very thick,-]
\tikzstyle{rededge}  =[draw=white,double=red,very thick,-]
\tikzstyle{prior} =[rectangle, draw=black, fill=black, minimum size=
5pt, inner sep=0pt]
\tikzstyle{dirprior} = [circle, draw=black, fill=black, minimum
size=5pt, inner sep=0pt]

\tikzstyle{dot_node}=[draw=black,fill=black,shape=circle]

\newlist{myenum}{enumerate}{1}
\setlist[myenum,1]{%
  label=\arabic*.,     
  leftmargin=*,        
  labelsep=0.5em,
  itemsep=2pt,
  topsep=4pt
}

\newlength{\hyplabelwidth}
\newenvironment{hypothesis}[1]{%
  \par\medskip\noindent
  \hangindent=\dimexpr\hyplabelwidth+0.5em\relax
  \hangafter=1
  \makebox[\hyplabelwidth][r]{#1:}\hspace{0.5em}%
  \ignorespaces
}{%
  \par\medskip
}

\setlist[itemize]{noitemsep,leftmargin=*,nosep}
\setlist[enumerate]{noitemsep,leftmargin=*,nosep}

\title{Falsifying Causal Graphs With Outlier Events}

\author[1]{\href{mailto:<will.r.orchard@gmail.com>?Subject=Your UAI 2026 paper}{William Roy Orchard$^*$}{}}
\author[2,3]{Philipp~M.~Faller$^*$}
\author[3]{Dominik Janzing}
\affil[1]{%
    University of Cambridge\\
    Cambridge, UK
}
\affil[2]{%
    Karlsruhe Institute of Technology\\
    Karlsruhe, Germany
}
\affil[3]{%
    Amazon Research\\
    Tübingen, Germany
  }
  
\begin{document}

\maketitle
\begingroup
\renewcommand{\thefootnote}{*}
\footnotetext{Authors contributed equally.}
\endgroup

\begin{abstract}%
True causal relationships are rarely known, and inferring causal graphs from data is hard. A fundamental challenge is how to assess whether a given causal graph is good in the absence of a ground truth. We propose falsifying candidate causal graphs based on whether they can explain the propagation of an outlier event. Our approach leverages a key principle: weak outliers rarely cause strong ones. While this principle has previously been used in root cause analysis to identify root causes without prior knowledge of the graph, we turn it on its head and use it to falsify candidate causal graphs whose implied outlier propagation is inconsistent with the data. To this end, we present the first statistical tests for the hypothesis that a candidate graph is the true causal graph, and show they have false positive control, power guarantees against incorrect causal graphs, and can operate with a single outlier sample.
\end{abstract}

\section{Introduction}\label{sec:introduction}
\looseness=-1Reasoning about causal relationships lies at the core of understanding and analysing complex systems. 
Unlike purely correlational methods, causal methods aim to leverage the underlying generative mechanisms that govern data, enabling robust prediction, counterfactual reasoning, and principled decision-making \citep{Pearl:00,spirtes2000causation}. 
 Despite significant progress on learning causal models from observational \citep{kalisch2007estimating,chickering2002optimal,shimizu2006linear,zheng2018dags,lam2022greedy} and interventional data \citep{heinze2018invariant, brouillard2020differentiable} and combinations of statistical data and expert knowledge \citep{ankan2025expert,hiremath2025guess2graph}, key challenges remain to reliably obtain a causal graph.
 Moreover, in practical settings it is often not clear how reliable a causal graph is, either because the given expert knowledge might be incomplete, or due to violated assumptions or finite-sample issues in a learned graph.
In light of this, there has been a recent surge in interest for methods to evaluate or falsify estimated causal quantities or models \citep{peters2016,canonne2017testing,heinze2018invariant,acharya2018learning,eigenmann2020evaluation,choo2022verification,bhattacharya2022testability,schkoda2024cross,eulig2025toward,faller2025different,faltenbacher2025internal}.
In this work we present a number of statistical tests for the null hypothesis that a given causal graph is the true causal graph.
To this end, we will leverage tools developed for another causal downstream task: root cause analysis.

\looseness=-1The goal of root cause analysis is to pinpoint the cause of some rare or anomalous event, e.g. in monitored cloud services \citep{hardt24a} or manufacturing \citep{gobler2024texttt}.
The main intuition behind our test is that, in a certain sense, weak outliers rarely cause strong outliers \citep{orchard2025root}.
\citet{orchard2025root} have proposed to use this to find root causes without prior knowledge about the underlying causal graph.
In this work we turn this principle on its head: we propose to falsify causal graphs if the observed 
propagation of outliers does not match this principle.

As far as we are aware, we present the first statistical tests for the null hypothesis that a given causal graph is the true causal graph that have:
\begin{itemize}
    \item false positive control;
    \item power against a broad class of incorrect causal graphs;
    \item and can operate with only a single outlier sample.
\end{itemize}

\looseness=-1The remainder of the paper is organised as follows. In~\cref{sec:related_work} we discuss related work. In~\cref{sec:prelims} we give the preliminaries 
to causal models and introduce information-theoretic outlier scores, upon which our tests are based, before formalising the problem setting. In~\cref{sec:falsification} we present our proposed tests along with their theoretical guarantees. We start with tests based on conditional outlier scores and show they have false positive control when the causal graph is a DAG. We are able to adapt some of the proposed tests to the setting where only \textit{marginal} outlier scores are available, and show they retain false positive control when the causal graph is a polytree. By deriving lower bounds on the rate of decay of outlier scores along causal paths, we then provide power guarantees for our tests. Finally, in~\cref{sec:experiments} we empirically compare our methods to several baselines on synthetic and real-world data.

\section{Related work}
\label{sec:related_work}
In what follows, we model outliers as being the result of a mechanism change. 
As such, they play a similar role to \enquote{natural experiments} \citep{angrist1996identification} and we can leverage invariances over different environments to facilitate falsification. 
The idea that different components of a causal model change independently has been used under different names and in different contexts, e.g., \enquote{the principle of independent mechanisms} \citep{haavelmo1944probability,aldrich1989autonomy,hoover1990logic,Pearl:00,dawid2010identifying,scholkopf2012causal,peters2017elements}. 
It is well-known that data from different environments or regimes can render causal quantities identifiable that are unidentifiable from i.i.d. data alone \citep{tian2013causal,ghassami2018multi,guo2022causal,huang2017behind,huang2020causal,lagani2012learning,mooij2020joint,perry2022causal,ke2023neural,schkoda2026root}, especially when the regimes correspond to interventions with known intervention targets \citep{Eberhardt2007,hauser2012characterization, yang2018characterizing}.
Further, such data can also improve statistical estimation of causal quantities as data from several regimes can be leveraged \citep{peters2016,heinze2018invariant,pfister2019invariant} or even estimate quantities for regimes that have not been observed \citep{shpitser2008complete,bareinboim2012transportability,bareinboim2013meta,pearl2011transportability}.
Using interventional data is often considered the gold-standard for graph evaluation.
Accordingly, previous work has studied the minimum number of interventions needed to falsify a graph \citep{hauser2014two,he2008active,shanmugam2015learning,choo2022verification}.
However, these methods assume access to a conditional independence oracle, whereas we demonstrate that already a single sample from a different regime can constitute enough evidence to reject a causal hypothesis. In applications such as cloud computing \citep{hardt24a} or manufacturing \citep{gobler2024texttt} it can be prohibitively expensive to gather multiple samples from an anomalous event, and, in personalised medicine \citep{li2025root} it may not be possible. 

\looseness=-1There has also been considerable work on falsifying causal models using observational data alone.
In \citep{faller2024self}, a causal discovery algorithm is run on different subsets of variables and the resulting graphs are checked for mutual consistency. Incompatibilities can indicate violated model assumptions or finite-sample errors. 
In \citep{schkoda2024cross} causal models learned on variable subsets are used to predict dependencies between variables that were omitted during training; large prediction errors suggest incorrect causal structure. \citet{faltenbacher2025internal,faller2025different} focus on constraint-based methods and use either additional or already conducted tests to evaluate a graph. None of these methods directly entails a valid statistical test. In contrast, \citet{eulig2025toward} propose a permutation-based procedure that enables formal testing against a random-baseline null. \citet{canonne2017testing,schkoda2025goodness} develop goodness-of-fit tests for Bayesian networks; the latter considers linear models with non-Gaussian noise and tests whether the implied distribution matches the data. Finally, \citet{li2009controlling,strobl2019estimating} study how edge-wise false positive control can be established in causal discovery, but they do not provide control for directions or absence of edges.

\section{Preliminaries}\label{sec:prelims}
We assume the causal relationships between $n\in\N$ variables $X_1,\dots,X_n$ are given by a 
Causal Bayesian Network (CBN) \citep{Pearl:00}, so that their joint distribution factorises into 
 conditionals according to an underlying DAG $\cG$:
\begin{equation}\label{eq:jointdist}
    P_{X_1,\dots,X_n} = \prod_{i = 1}^{n}P_{X_i\mid \PA^\cG_i},
\end{equation}
where each $P_{X_i\mid \PA^\cG_i}$ corresponds to the causal mechanism of variable $X_i$ given its parents $\PA^\cG_i$ in $\cG$.

Throughout this work, we will deal with outliers.
Formally, we think of outliers as being the result of a soft-intervention \citep{Eberhardt2007}: let $\cR\subseteq \{1,\dots, n\}$ be a non-empty index set, we say outlier sample $(x_1,\dots,x_n)$ is drawn from the `corrupted' distribution,
\begin{equation}\label{eq:jointdistcorrupted}
    \tilde{P}_{X_1,\dots,X_n} = \prod_{j\in\cR}\tilde{P}_{X_j\mid \PA^\cG_j}\prod_{i\notin \cR} P_{X_i\mid \PA^\cG_i},
\end{equation}
where for each of the so-called \emph{root cause} variables $X_j$, with index $j\in\cR$, we assume the `ordinary' mechanism 
$P_{X_j\mid\PA^\cG_j}$ has been replaced for a `corrupted' one $\tilde{P}_{X_j\mid\PA^\cG_j}$. Note that a root \emph{cause} need not be a root \emph{node}.
Next, we define \emph{information theoretic (IT) outlier scores} as proposed by 
\citet{root_cause_analysis}.

Let $\tau_i: \cX_i \rightarrow \R$ be a feature map, for example, any existing real-valued
anomaly scoring function.
Define the event
\begin{equation}\label{eq:event}
    E_{x_i}:=\{\tau_i(X_i)\geq \tau_i(x_i)\},
\end{equation}
of a sample being at least as extreme as the observed event $x_i$ w.r.t. $\tau$.
The \emph{marginal IT outlier score} is defined by
\begin{displaymath}\label{eq:outlier_score}
    S_{X_i}^{\tau_i}(x_i) := -\log P(\tau_i(X_i) \geq \tau_i(x_i)) = -\log P(E_{x_i}).
\end{displaymath}
Note that each variable may have a different feature map, but we usually drop $X_i$ and $\tau_i$
from $S_{X_i}^{\tau_i}$ whenever it is clear from context which variable and feature map we refer to.

As pointed out by 
\citet{orchard2025root}, we 
can interpret $P(E_{x_i}) = e^{-S(x_i)}$ as a p-value for the null hypothesis that $x_i$ was drawn from 
its ordinary marginal distribution.
We also define the \emph{conditional} IT outlier score,
\begin{displaymath}\label{eq:conditionalscore}
    S(x_i\mid \pa^\cG_i) := -\log P(\tau_i(X_i)\geq \tau_i(x_i)\mid \PA^\cG_i = \pa^\cG_i).
\end{displaymath}
Similarly, $P(E_{x_i}\mid \PA^\cG_i=\pa^\cG_i) = e^{-S(x_i\mid \pa^\cG_i)}$ is a p-value for the 
null hypothesis that $x_i$ was drawn from its ordinary conditional distribution given the values of its parents 
$\pa^\cG_i$ in $\cG$. A large conditional outlier score corresponds to a small p-value and so we 
favour rejecting the null in such cases.

\paragraph{Problem statement:} given an estimated causal graph $\hat \cG$, test the null 
hypothesis $H_0: {\hat\cG} = \cG$, using the outlier sample $(x_1,\dots,x_n)$, by leveraging theory for the 
propagation of outliers in causal systems.

\section{Falsifying causal graphs with IT outlier scores}\label{sec:falsification}
In this section we propose tests for the null hypothesis $H_0$. 
The tests are stated under the assumption that we \textit{know} the root cause(s) of the outlier event, but we will close 
this section by showing this assumption may be dropped in the case one knows an upper bound on the number of root causes. In each instance we will
state a test statistic $T: \cX\to \R$ and corresponding p-value $p_T$ for $H_0$. It should be understood 
that the corresponding level-$\alpha$ test $\phi_T: \cX \to \{0,1\}$ is defined in the usual way: 
$\phi_T(x) := \bI[p_{T(x)} \leq \alpha]$. Recalling that a p-value for $H_0$ is a random variable $p$ that 
satisfies $P_{H_0}(p\leq \alpha)\leq\alpha$, the immediate consequence of identifying a p-value for $H_0$ is
that the corresponding test controls the false positive rate. As such, we will not restate this 
implication for each p-value we derive. 

All proofs, unless otherwise indicated, are given 
in \cref{app:proofs} of the appendix.

\subsection{Falsifying causal graphs with conditional outlier scores}
To test our primary null hypothesis $H_0$ we propose testing another 
hypothesis that it implies:
\begin{hypothesis}{$H_0^{\cR}$}
    for outlier sample $(x_1,\dots, x_n)$, all $x_i$ with $i\notin\cR$ were drawn according to $P(X_i\mid\PA_i^{\hat{\cG}} = \pa_i^{\hat{\cG}})$.
\end{hypothesis}

In other words, the outlier sample, excluding the root causes, was drawn according to the causal 
conditionals implied by the \emph{estimated} graph $\hat{\cG}$. If $H_0$ is true, the causal
conditionals implied by $\hat{\cG}$ coincide with the true causal mechanisms, and so $H_0^{\cR}$
is true by definition. Rejecting $H_0^{\cR}$ is therefore sufficient to reject $H_0$.\footnote{Note that this assumes that the outlier sample is not selected because it is anomalous and then also used for testing, as this would introduce selection bias.}

To define our test statistics, we will need the assumption:
\begin{enumerate}
    \item {\bf Continuity:}
        every $P(\tau_i(X_i)\mid \PA^{\cG}_i = \pa^{\cG}_i)$ is absolutely continuous with respect to the Lebesgue measure.
\end{enumerate}
Conditional IT scores with variable inputs $X_i, \PA_i^\cG$
define random variables $S(X_i\mid\PA_i^\cG)$. We then have (see also lemma 3.7 of \citet{orchard2025root}),
\begin{lemma}[{\bf Conditional scores are independent standard exponentials}]\label{lem:scores_exponential_under_null}
    Subject to continuity, for variables $X_i$ distributed according to $P$, with $i\in\{1,\dots,n\}$, $S(X_i\mid\PA_i^\cG)$ are independent, 
    ${\rm Exp}(1)$ distributed random variables.
\end{lemma}
\looseness=-1This property underpins each of the tests we propose based on conditional scores.
With this in place, the first three of our tests are based on the same idea: high conditional outlier 
scores with respect to $\hat{\cG}$ are evidence for rejecting $H_0^{\cR}$ and therefore for rejecting $H_0$. As the number of non-root cause variables will appear in
many equations below, we will denote it $L := n - |\cR|$ to make presentation less cumbersome.

The first test statistic is the sum of conditional outlier scores of the non-root causes 
$S^{\hat{\cG}}_{\rm sum} := \sum_{i\notin\cR}S(X_i\mid\PA_i^{\hat{\cG}})$.
The sum of $m$ independent standard exponential random variables is Erlang distributed
with scale parameter $1$ and shape parameter $m$ \citep{Ibe2013}. As such, we can show that
\begin{proposition}[{\bf Fisher-based test}]\label{prop:fishertest}
    $H_0$ can be rejected with p-value
    \begin{equation}
        p_{S^{\hat{\cG}}_{\rm sum}} = e^{-S^{\hat{\cG}}_{\rm sum}}\cdot\sum_{l=0}^{L-1}\frac{(S^{\hat{\cG}}_{\rm sum})^l}{l!}.
    \end{equation}
\end{proposition}
This corresponds to using Fisher's method \citep{Fisher1925} for the combination of the independent p-values $e^{-S(X_i\mid\PA_i^{\hat{\cG}})}$.

The second test statistic simply uses the maximum conditional outlier score from among the non-root causes,
$S_{\rm max}^{\hat{\cG}} := \max_{i\notin\cR} S(X_i\mid\PA_i^{\hat{\cG}})$. This too has a simple
distribution under the null so that
\begin{proposition}[{\bf Tippett-based test}]\label{prop:tippetttest}
    $H_0$ can be rejected with p-value
    \begin{equation}
        p_{S_{\rm max}^{\hat{\cG}}} = 1 - (1 - e^{-S_{\rm max}^{\hat{\cG}}})^{L}.
    \end{equation}
\end{proposition}
This corresponds to using Tippett's method \citep{Tippett1931} for the combination of independent p-values.

The third test statistic counts the number of conditional outlier scores, from among the non-root causes, that exceed a given 
threshold $t$: $N_t := \sum_{i\notin\cR} \mathbb{I}[S(X_i\mid\PA_i^{\hat{\cG}}) \geq t]$. As the probability 
that a standard exponential random variable exceeds a value $t$ is simply $e^{-t}$, the corresponding p-value
is the upper-tail of the Binomial distribution with probability of success $e^{-t}$ and $L$ 
trials:
\begin{proposition}[{\bf Binomial test on scores above threshold}]\label{prop:binomialtest}
    $H_0$ can be rejected with p-value
    \begin{equation}
    p_{N_t} = \sum_{k = N_t}^{L}\binom{L}{k} e^{-kt}(1 - e^{-t})^{L - k}.
    \end{equation}
\end{proposition}
As all three tests have false positive control, a remaining question is when one should
prefer one test over another, i.e. when does each test have power. We will devote \cref{subsec:power}
to a thorough analysis of power, but what is common to all three is that they can only have power
when errors in $\hat{\cG}$ result in inflated outlier scores. As we will see, certain errors in $\hat{\cG}$
do indeed result in inflated outlier scores, but nonetheless we may be interested in
a test which has power against more general deviations from the null. As such let 
$F_{\rm Exp}(s) := \left(1 - e^{-s}\right)\bI[s\geq0]$ denote the CDF of the standard exponential distribution and
\begin{equation}
    F_{L}(s) := \frac{1}{L}\sum_{i\notin\cR} \bI[S(X_i\mid\PA_i^{\hat{\cG}}) \leq s]
\end{equation}
be the empirical CDF of the conditional outlier scores of the non-root causes. Our fourth test statistic is
\begin{equation}
    D_{L} := \sup_{s}\left|F_{L}(s) - F_{\rm Exp}(s)\right|.
\end{equation}
Let $K_n$ be the null distribution function of the Kolmogorov-Smirnov statistic for sample size $n$ \citep{Massey1951}, then
\begin{proposition}[{\bf Kolmogorov-Smirnov test}]\label{prop:kstest}
    $H_0$ can be rejected with p-value
    \begin{equation}
        p_{D_L} = 1 - K_{L}(D_{L}).
    \end{equation}
\end{proposition}
\begin{proof}
    An immediate consequence of
    \cref{lem:scores_exponential_under_null} is that, under $H_0$, $F_{L}$ is the empirical CDF of the standard
    exponential for sample size $L$. $D_{L}$ is therefore the 
    corresponding Kolmogorov-Smirnov statistic and our test is the Kolmogorov-Smirnov test.
\end{proof}

\subsection{Falsifying causal graphs with marginal outlier scores}
Inferring the causal conditionals $P(X_i\mid\PA_i^{\hat{\cG}})$, and therefore the corresponding 
conditional outlier scores, is statistically ill-posed, especially when the number of parents is large.
One would therefore ideally avoid it.
In this section we show how the first three of our tests can be adapted to use only \emph{marginal} outlier 
scores. However, to do so we need to introduce results on the propagation of outliers.

Let $\cG$ be $X\to Y$ and $(x, y)$ be the observed outlier sample. Following \citet{orchard2025root} we give the definition:
\begin{definition}[{\bf Score typicality}]
    We say observation $(x,y)$ of $(X,Y)$ satisfies score typicality if
    \begin{equation}
        S(y\mid x) \geq |S(y) - S(x)|_+,
    \end{equation}
    where $|a|_+ := \max(0, a)$.
\end{definition}
The significance of score typicality is that, when satisfied, it implies a key principle
of outlier propagation: weak outliers rarely cause strong outliers. To see this, note that
subject to score typicality
\begin{align}
    P(S(Y)\geq S(y)\mid X = x) = e^{-S(y\mid x)} \leq e^{-|S(y) - S(x)|_+},
\end{align}
so that if $y$ was drawn according to its ordinary causal mechanism, $S(y) \gg S(x)$ 
is exponentially unlikely. That it is justified to refer to this condition as
`typical' is given by the following result from \citet{orchard2025root}:
\begin{lemma}[{\bf Score typicality probably holds, approximately}]\label{lem:score_typical_pac}
    For any given $y$, and for any $s_X \leq S(y) - 1$, let (possibly multivariate)
    $x$ be chosen at random according to $P(X\mid S(X)\in [s_X, S(y)])$. Subject to
    continuity, with probability at least $1 - 2/c$, we obtain an $x$ for which
    \begin{equation}
        S(y\mid x) \geq |S(y) - S(x)|_+ - 2\log c.
    \end{equation}
\end{lemma}
In other words, for any $y$, although score typicality may not hold for any
\emph{particular} $x$, \cref{lem:score_typical_pac} shows that it is likely to 
hold approximately for a randomly chosen $x$. In what follows, we will state certain
results \emph{subject to score typicality} for brevity, but in such cases it should
be remembered that we could drop this condition at the cost of slightly weaker bounds
holding for most samples.

\Cref{lem:score_typical_pac} applies without modification to general DAGs
by considering each $(X_i, \PA_i^{\cG})$ in turn, in place of $(Y, X)$, though note that for multivariate $\PA_i^\cG$, the associated feature map for $S(\PA_i^\cG)$
is, correspondingly, multivariate too. When restricting ourselves only to the inference
of marginal distributions we add the assumption:
\begin{enumerate}
    \setcounter{enumi}{1}
    \item {\bf Polytree:}
        The true causal graph $\cG$ is a polytree.\footnote{See \citet{li2025root, ebtekar2025toward} for why it is not straightforward to drop this assumption for marginal scores.}
\end{enumerate}
A polytree is a DAG whose skeleton contains no undirected cycles. As a consequence, 
the parents of any given variable are independent, and we may combine their
marginal scores, akin to in \cref{prop:fishertest,prop:tippetttest,prop:binomialtest},
to give a joint score. For example, \citet{orchard2025root} show that, subject to continuity, for variable $X_i$ with parents $\PA_i^1, \dots, \PA_i^k$, 
and (joint) event $\pa_i$, the following defines a joint IT outlier score:
\begin{equation}\label{eq:jointscore}
    S(\pa_i) := \sum_{j=1}^{k} S(\pa_i^j) - \log \sum_{l=0}^{k-1} \frac{(\sum_{j=1}^{k} S(\pa_i^j))^l}{l!}.
\end{equation}
With this in place we are ready to give the marginal score analogues of our first three
tests. Let
\begin{align}
    \hat{S}^{\hat{\cG}}_{\rm sum} &:= \sum_{i\notin\cR}|S(X_i) - S(\PA_i^{\hat{\cG}})|_+, \\
    \hat{S}_{\rm max}^{\hat{\cG}} &:= \max_{i\notin\cR}|S(X_i) - S(\PA_i^{\hat{\cG}})|_+, \\
    \hat{N}_t &:= \sum_{i\notin\cR} \mathbb{I}[|S(X_i) - S(\PA_i^{\hat{\cG}})|_+ \geq t],
\end{align}
be the marginal score analogues of each of our first three test statistics. Then
\begin{proposition}[{\bf Marginal score tests}]\label{prop:marginaltests}
    Subject to score typicality, each of the following
    are p-values for $H_0$:
    \begin{align}
        p_{S^{\hat{\cG}}_{\rm sum}} &\leq p_{\hat{S}^{\hat{\cG}}_{\rm sum}}= e^{-\hat{S}^{\hat{\cG}}_{\rm sum}}\cdot\sum_{l=0}^{L-1}\frac{(\hat{S}^{\hat{\cG}}_{\rm sum})^l}{l!}, \\
        p_{S_{\rm max}^{\hat{\cG}}}&\leq p_{\hat{S}_{\rm max}^{\hat{\cG}}} = 1 - (1 - e^{-\hat{S}_{\rm max}^{\hat{\cG}}})^{L}, \\
        p_{N_t}&\leq p_{\hat{N}_t} = \sum_{k = \hat{N}_t}^{L}\binom{L}{k} e^{-kt}(1 - e^{-t})^{L - k}.        
    \end{align}
\end{proposition}
In each case, the proposition follows from the marginal score test statistics each being
stochastically dominated by their conditional outlier score counterparts.

\subsection{Bounding marginal score jumps}\label{subsec:boundmarginal}
When, according to $\hat{\cG}$, we observe a weak outlier seemingly causing a strong outlier, there is a marginal score `jump' from cause to effect. We saw in the previous section that all of our marginal score test statistics increase when there are such jumps, and therefore the corresponding p-values decrease. As such, we expect our marginal score tests to have power precisely when errors in $\hat{\cG}$ result in marginal score jumps. In this section, we show when this happens.

To this end, let us add a faithfulness-type assumption \citep{Jaber2020} that says that the outlier at the root cause(s) $X_{\cR}$
induces outliers in all of its descendants:
\begin{enumerate}
    \setcounter{enumi}{2}
    \item {\bf Effectiveness:}
        There exists a $\lambda > 0$, such that for all variables $X_j$ among the descendants
        of $X_{\cR}$ in $\cG$, $S(X_j) \geq \lambda$.
\end{enumerate}
The causal mechanisms of the
non-descendants of the root causes work as normal, and are not influenced by the outlier event, so will 
generally have small outlier scores. Therefore, when $\lambda$ is large, if a (true) descendant of the root causes has been 
assigned non-descendants as parents in $\hat{\cG}$ (or no parents), the marginal score jump will be at 
least $\approx \lambda$. 

That we could have power to reject graphs that do not correctly place the descendants of root
causes is not surprising when we consider that root causes are equivalently thought of as the targets of a soft
intervention. 
Indeed, if the causal model is deterministic, in the worst case, all descendants of the root cause may receive 
identical marginal outlier scores and so it will only be possible to distinguish descendants from non-descendants.
However, in general, we expect outliers to decay as they propagate further away from root causes. In such cases, marginal score jumps will also be introduced when (true) descendants of root causes are assigned parents, in $\hat{\cG}$, that are \emph{further away} from the root causes than they are. To formalise this intuition we need to introduce some concepts 
from information theory (see \citet{PolyanskiyWu2025} for each definition below). 

The Kullback-Leibler (KL) divergence for probability distributions $P$ and $Q$, where $P\ll Q$, is given by
\begin{equation}\label{eq:kldivergence}
\KL{P}{Q} := \Ex[P]{\log\frac{dP}{dQ}}.
\end{equation}

For a fixed channel $P_{Y\mid X}$, we denote by $P_{Y\mid X}\circ P$ the distribution induced by the push-forward of 
the distribution $P$, i.e. the distribution on the output $Y$ when the input $X$ is distributed according to $P$.
The Strong Data Processing Inequality (SDPI) with respect to the KL-divergence is
\begin{equation}\label{eq:sdpi}
\KL{\chan\circ P}{\chan\circ Q} \leq \eta(\chan, P_X)\,\KL{P}{Q},
\end{equation}
where $\eta(\chan, P_X) \in [0, 1]$ is the \emph{contraction coefficient} for channel $\chan$ and fixed input $Q$,
and is defined as:
\begin{equation}\label{eq:contractioncoeff}
    \eta(P_{Y\mid X}, Q) := \sup_{\mathmakebox[0.8\width][c]{P: 0 < \KL{P}{Q} < \infty}} \frac{\KL{\chan\circ P}{\chan\circ Q}}{\KL{P}{Q}}.
\end{equation}
For brevity we will abbreviate $\eta := \eta(P_{Y\mid X}, Q)$ when the channel and input are clear from context. As stated, however, the SDPI is not interesting unless $\eta < 1$. There are many channels where this is the case. We give two examples 
\begin{proposition}[{\bf Linear Gaussian model}]\label{prop:lineargaussian}
    For causal model $Y := \beta X + N$, $N \ind X$ with $(X, Y)$ jointly Gaussian and Pearson
    correlation coefficient $\rho$, 
    \begin{equation}
        \eta = \rho^2.
    \end{equation}
\end{proposition}
\begin{proof}
    Immediate consequence of Theorem 5 of \citet{Makur2020}.
\end{proof}

\begin{proposition}[{\bf Bounded Gaussian additive noise model}]\label{prop:boundanm}
    For causal model $Y := f(X) + N$, $N \ind X$, and there exists a $B\geq 0 $ such that $|f(x)| \leq B$ for all $x$, and
    $N\sim \mathcal{N}(0, \sigma^2)$,
    \begin{equation}
        \eta \leq 2\,\Phi(B/\sigma) - 1,
    \end{equation}
    where $\Phi$ is the CDF for the standard Gaussian distribution.
\end{proposition}


We can put the SDPI to use by noting
\begin{lemma}[{\bf Marginal outlier scores are KL-divergences}]\label{lem:marginalscoreKL}
For variable $X$ and outlier event $E_x$, for $P_X(E_x) > 0$,
\begin{equation}
S(x) = \KL{P_{X\mid E_x}}{P_X}.
\end{equation}
\end{lemma}
Suppose in $\cG$ that $X\to Y$ (where $X$ is the possibly multivariate parent set of $Y$), and we observe outlier sample $(x, y)$ with corresponding outlier events $E_x, E_y$ as in \cref{eq:event}. Suppose also that $Y$ is not a root cause. As $Y$ is not a root cause, its mechanism $P_{Y\mid X}$ worked as normal and 
so we expect the conditional outlier score for $y$ to be small. Applying \cref{lem:marginalscoreKL} to the SDPI, when $S(y\mid E_x) = 0$, we find
\begin{lemma}[{\bf Lower bound on the rate of outlier decay}]\label{lem:marginaldecay}
\begin{equation}
\eta S(x) \geq S(y).
\end{equation}
\end{lemma}
In other words, when the outlier is propagated by the `working' causal mechanism, $P_{Y\mid X}$, it decays by at least a 
factor of $\eta$. Though, for brevity, we have stated the result given $S(y\mid E_x) = 0$, we show in \cref{app:marginaldecay} of the appendix that, for any given $x$, for a random $y$ drawn according to $P_{Y\mid E_x}$, \cref{lem:marginaldecay} is likely to hold, approximately.

This has immediate consequences for characterising the power of our marginal tests. Suppose in $\hat{\cG}$ the relationship
between $X$ and $Y$ is \emph{incorrectly} given as $Y\to X$, then when $S(y\mid E_x) = 0$,
\begin{lemma}[{\bf Anti-causal score jump lower bound}]\label{lem:anticausalbound}
    \begin{equation}
        |S(x) - S(y)|_+ \geq (1 - \eta)\,S(x).
    \end{equation}
\end{lemma}
\begin{proof}
    An immediate consequence of \cref{lem:marginaldecay}.
\end{proof}
If $Y$ is not a root cause, and $X$ is a true descendant of a root cause, then the score jump will be at least $(1-\eta)\lambda$. We have therefore established two kinds of errors in $\hat{\cG}$ that imply score jumps, both of which concern assigning the wrong parent set to a true descendant of a root cause.


\subsection{Bounding conditional scores}\label{subsec:boundconditional}
We have established the foundation for providing power guarantees for our marginal score tests. We now show we can do the same for conditional outlier scores.

\looseness=-1Let $X\to Y$ in $\cG$ where $X$ is the (possibly multivariate) parent set of $Y$, let $(x,y)$ be the outlier sample, and suppose the causal relationship is incorrectly given as $Y\to X$ in $\hat{\cG}$. An interesting case is when large marginal score jumps imply large conditional scores, allowing us to simply apply the results from the previous section to conditional scores directly:
\begin{lemma}[{\bf Monotonicity implies anti-causal conditional score lower bound}]\label{lem:monotonimpliesbound}
    If $S_Y$ is injective, and $P(S(X)\geq S(x)\mid S(Y) = S(y)) \leq P(S(X)\geq S(x)\mid S(Y)\geq S(y))$, then
    \begin{equation}
        S(x\mid y) \geq |S(x) - S(y)|_+.
    \end{equation}
\end{lemma}
This uses a type of monotonicity property. It says that observing a more extreme outlier at $Y$ should only increase our belief that the outlier at $X$ is more extreme too. In this case, whenever our marginal score tests have power, so do our conditional score tests. This monotonicity property holds for a wide class of causal models. For example,
\begin{proposition}[{\bf Additive log-concave noise model}]\label{prop:anmlogconcave}
For real-valued variables where $Y := f(X) + N$, $N\ind X$, $N$ is log-concave, $f$ is non-decreasing, and $S_X, S_Y$ are strictly increasing then the monotonicity of \cref{lem:monotonimpliesbound} is satisfied for all $(x, y)$.
\end{proposition}

However, inheriting the marginal score results is not the only way to establish when errors in $\hat{\cG}$ imply inflated conditional scores, and so we do not need to rely on the monotonicity property. Firstly, there are two trivial cases. One, if a true descendant of a root cause is assigned no parents in $\hat{\cG}$, then its conditional score will simply be its marginal score and as such, subject to effectiveness, will be at least $\lambda$. Two, if a true descendant of a root cause, call it $X_i$, is assigned parents $\PA_i^{\hat{\cG}}$ in $\hat{\cG}$ where $X_i \ind \PA_i^{\hat{\cG}}$, then again, subject to effectiveness, $S(x_i\mid \pa_i^{\hat{\cG}}) = S(x_i) \geq \lambda$.

The first non-trivial case is when a true descendant of a root cause $X_i$ is assigned parents $\PA_i^{\hat{\cG}}$ in $\hat{\cG}$ that are non-descendants of root causes (and not root causes themselves), but where $X_i \not\ind \PA_i^{\hat{\cG}}$. In this case we find
\begin{lemma}[{\bf Score lower bound after conditioning on non-descendant}]\label{lem:boundconditionalnondesc}
    Subject to effectiveness, let $(x_i, \pa_i^{\hat{\cG}})$ be drawn according to the corrupted distribution $\tilde{P}$, then with probability at least $1 - 1/c$
    \begin{equation}
        S(x_i\mid \pa_i^{\hat{\cG}}) \geq \lambda -\log c.
    \end{equation}
\end{lemma}
In this case, while $S(x_i \mid \pa_i^{\hat{\cG}}) \geq \lambda$ need not hold for all $(x_i, \pa_i^{\hat{\cG}})$, it does hold approximately, with high probability, for a random outlier sample.

The other non-trivial case is akin to that discussed for marginal scores in \cref{lem:anticausalbound}. As before, let $X\to Y$ in $\cG$ and suppose that in $\hat{\cG}$ the causal relationship is incorrectly given as $Y\to X$, we have
\begin{lemma}[{\bf Expected anti-causal conditional score bound}]\label{lem:lowboundexpect}
    \begin{equation}
        \bE_{P_{Y\mid E_x}}[S(x\mid Y)] \geq (1-\eta)\,S(x).
    \end{equation}
\end{lemma}
While this only gives a bound on the \emph{expected} conditional score, so long as $S(x\mid Y)$ is concentrated near its expectation under $Y\sim P_{Y\mid E_x}$, then this implies a bound on $S(x\mid y)$ for typical $y$. This is true for a wide class of causal models. A straightforward example is the class of models for which ${\rm Var}_{P_{Y\mid E_x}}[S(x\mid Y)] =: V_x < \infty$, as applying Cantelli's inequality \citep{Ghosh2002} we can conclude
\begin{lemma}[{\bf Anti-causal conditional score lower bound}]\label{lem:conditionallowbound}
    For any $x$, let $y$ be chosen at random according to $P_{Y\mid E_x}$. With probability at least $1 - \delta$, we obtain a $y$ for which
    \begin{equation}
        S(x\mid y)\geq (1-\eta)\,S(x) - \sqrt{\frac{1 - \delta}{\delta}V_x}.
    \end{equation}
\end{lemma}
The following lemma gives examples of causal models for which $V_x$ can be bounded so that \cref{lem:conditionallowbound} applies,
\begin{proposition}[{\bf Bounded additive noise model}]\label{prop:anmboundedvariance}
    For causal model $Y := f(X) + N$, $N\ind X$, and there exists a $B\geq 0$ such that $|f(x)|\leq B$ for all $x$,
    \begin{enumerate}
        \item If $N$ has a positive density whose logarithm is Lipschitz with constant $L_N$:
        \begin{equation}
            V_x \leq (2L_NB)^2.
        \end{equation}
        This includes the Cauchy, Laplace, Logistic and Student's $t$ distributions.
        \item If $N\sim \cN(0, \sigma^2)$ :
        \begin{equation}
            V_x \leq \frac{4B^2}{\sigma^4}(B^2 + \sigma^2) + \frac{2B^3}{\sigma^4}\sqrt{B^2 + \sigma^2} + \frac{B^4}{4\sigma^4}.
        \end{equation}
    \end{enumerate}
\end{proposition}
In the case $N$ has a density that is log-Lipschitz, we actually derive a non-probabilistic bound on $S(x\mid y)$ that in many cases is stronger than that given by \cref{lem:conditionallowbound} (see proof of proposition 19 in the appendix for details). For brevity, we present only the bound above in the main text.

\subsection{Power guarantees}\label{subsec:power}
In the previous two subsections, we laid the foundations for providing power guarantees for our tests:
\begin{remark}[{\bf Detectable errors}]\label{rem:errors}
    Our conditional score tests, and their marginal score analogues, are sensitive to at least two kinds of errors in $\hat{\cG}$:
    \begin{enumerate}
    \item assigning to a descendant of a root cause either no parents, or parents that, in $\cG$, are not descendants of a root cause. In this case, the marginal score jump (effectiveness) and the conditional score (effectiveness and \cref{lem:monotonimpliesbound} or \cref{lem:boundconditionalnondesc}) are lower bounded by $\approx\lambda$;
    \item assigning to a descendant of a root cause parents that, in $\cG$, are among its descendants. In this case, the marginal score jump (effectiveness and \cref{lem:anticausalbound}) and the conditional score (effectiveness and \cref{lem:monotonimpliesbound} or \cref{lem:conditionallowbound}) are lower bounded by $(1 - \eta_{\max})\lambda$,
\end{enumerate}
\looseness=-1where letting $\eta_i$ be the contraction coefficient for channel $P_{X_i\mid \PA^{\cG}_i}$, then $\eta_{\max} := \max_{i\notin\cR} \eta_i$. As $\lambda\geq (1 - \eta_{\max})\lambda$, we will refer to $b := (1 - \eta_{\max})\lambda$ as the \emph{lowest lower bound} on the conditional scores and marginal score jumps that are misspecified in $\hat{\cG}$ due to the two kinds of errors just listed.
\end{remark}
To see how each test is sensitive to these errors, let $\pi$ be the fraction of non-root causes whose parents in $\hat{\cG}$ are incorrect in either of the two ways listed above. Informally, $\pi$ can be thought of as the fraction of $\hat{\cG}$ that is `detectably' incorrect. We have the following bounds on their p-values, where, to avoid writing out each bound twice, we write $p_{\tilde{T}}$ to denote the p-value for \emph{either} the conditional score test with test statistic $T$ or the marginal score test with test statistic $\hat{T}$:
\begin{proposition}[{\bf Upper bounds on p-values}]\label{prop:pvaluebounds}
    Subject to the conditions given in \cref{rem:errors} and $\pi > 0$,
    \begin{align}
         p_{\tilde{S}^{\hat{\cG}}_{\rm sum}} 
          &\leq e^{-\pi L b}\sum_{l=0}^{L-1}\frac{(\pi L b)^l}{l!}, \\
        p_{\tilde{S}_{\rm max}^{\hat{\cG}}} 
          &\leq 1 - (1 - e^{-b})^{L}, \\
        p_{\tilde{N}_t} 
          &\leq 
      \begin{cases}
        \sum_{k=\pi L}^{L} \binom{L}{k} e^{-kt}(1-e^{-t})^{L-k}
          & \text{if } b \ge t, \\
        1 & \text{otherwise}.
      \end{cases}
    \end{align}
\end{proposition}
It is important to note that the bounds apply to the marginal score and conditional score test p-values under different conditions (see \cref{rem:errors}). Each of these bounds provides hints for when we expect one test to be preferable to the others. 
We plot them in \cref{sec:power_curves} of the appendix.
\looseness=-1While the bounds on both the first and second p-values decrease with increasing $b$, which is generally high when the outlier event is strong (because $\lambda$ is large), only the first decreases with increasing $\pi$. This makes sense as $S_{\rm max}^{\hat{\cG}}$ considers only the maximum score, so that even if there is only a single error in $\hat{\cG}$, so long as $b$ is large enough, the p-value will be small. In contrast, the first test may be preferable for more moderate outlier events when one expects $\hat{\cG}$ to have a larger proportion of errors. For the third p-value, so long as $b \geq t$, the bound decreases with increasing $\pi$ but is otherwise insensitive to $b$. However, the bound tends to decay more sharply with $\pi$ than for the first p-value, so long as $t$ is not chosen to be too much smaller than $b$.

\looseness=-1So far we have said little about the fourth test (\cref{prop:kstest}) as it is the only without a marginal score analogue, and the only for which the p-value is sensitive to more than increased conditional scores under the alternative. Nonetheless, we can give power guarantees for the test. Consider the alternative
\begin{hypothesis}{$H_1$}
$S(X_i\mid\PA_i^{\hat{\cG}})$, for $i\notin\cR$, are iid with CDF
\begin{equation}
    F(s) := \pi G(s) + (1 - \pi)(1-e^{-s}),
\end{equation}
\end{hypothesis}
where $G(s)$ is a CDF such that $G(t) = 0$.
This alternative models the case that a proportion $\pi$ of the conditional outlier scores are drawn from an unknown distribution, but are bounded away from zero by a value $t$, while the remaining are standard exponential. This almost matches our case of interest, for $t = b$, but with the additional simplifying assumption that the scores are independent under the alternative. We then have,
\begin{theorem}[{\bf Power lower bound on KS test}]
\label{thm:kstestpower}
Subject to the conditions given in \cref{rem:errors}
    \begin{equation}
        P_{H_1}(D_L \geq D_{L,\alpha}^{\rm crit}) \geq 1 - 2e^{-2L|\pi (1-e^{-b}) - D^{\rm crit}_{L,\alpha}|^2_+},
    \end{equation}
where $D^{\rm crit}_{L,\alpha}$ is the critical value of the KS-statistic for sample size $L$ and significance threshold $\alpha$, and $t = b$.
\end{theorem}

Beyond bounds on the power or p-values for each test, it can be instructive to consider for which alternatives each test is optimal, i.e. uniformly most powerful (UMP). In \cref{app:ump} of the appendix, we slightly adapt a result from \citet{HeardRubinDelanchy2018} to show in \cref{thm:optimalfisher} that the $S^{\hat{\cG}}_{\rm sum}$ test is UMP against an alternative wherein the conditional outlier scores remain independent, and exponentially distributed, but with common rate parameter $\nu < 1$. In other words, when each outlier score is stochastically larger than \dominik{maybe add 'its size' for better readibility? I got stuck at this sentence} under the null. This is consistent with our observation above that the test is preferable when there is a large proportion of errors in $\hat{\cG}$ even when the outlier is only moderately strong. In contrast, \cref{thm:optimaltippett} shows that the $S^{\hat{\cG}}_{\rm max}$ test is UMP against an alternative where a single score is drawn from a distribution which stochastically dominates the standard exponential, and the remaining are drawn from a standard exponential distribution, truncated to be less extreme than that first score. In this sense, the test is optimal when there is a single, dominating, large score. 

For the conditional score binomial test, the alternative is more directly comparable to our setting,
\begin{theorem}[{\bf Conditional score binomial test is uniformly most powerful}]\label{thm:binomialoptimal}
    The conditional score binomial test in \cref{prop:binomialtest} is uniformly most powerful against the alternative
    \begin{hypothesis}{$H^{\rm Exp}_1$}
        $S(X_i\mid \PA_i^{\hat{\cG}}),$\hspace{0.4em}for $i \notin\cR$, are iid with distribution $f(s) =\pi e^{-(s-t)}\mathbb{I}[s\geq t] + (1-\pi)e^{-s}$.
    \end{hypothesis}
\end{theorem}
\looseness=-1$H_1^{\rm Exp}$ is a special case of $H_1$ above where the distribution of $G(s)$ is taken to be the standard exponential distribution conditional on $S \geq t$. This suggests that the $N_t$ test is well suited to our setting so long as $t$ is close to $b$ (but doesn't exceed it). This matches the conclusions drawn from \cref{prop:pvaluebounds}.

\subsection{Falsifying causal graphs when the root cause(s) are unknown}\label{subsec:unknownroots}
So far we have presented our tests under the assumption that the root cause(s) of the outlier event are known. Now we show that this assumption may be dropped in the setting one only knows (an upper bound on) the \emph{number} of root causes.

Note that one cannot simply infer the root cause(s) with the outlier sample $(x_1,\dots, x_n)$, using a root cause analysis algorithm (e.g. one of the approaches in \citet{orchard2025root,schkoda2026root}), 
and then apply one of our tests with the inferred root cause(s). This introduces selection bias, so the remaining conditional outlier scores are no longer independent, and our tests no longer valid.\will{there is more to be said here as I don't think you could combine a p-value even from an independent source of data as we don't have p-values which are valid under the joint null - though I find it very confusing}

Suppose instead we know an upper bound $k \geq |\cR|$ on the number of root causes. We propose the following approach to test $H_0$. For every $\cC\subseteq \{1,\dots, n\}$ such that $|\cC| = k$, let $T$ denote any of the test statistics we have proposed and $p_T^{\cC}$ be the corresponding "p-value" after excluding the conditional outlier scores (or marginal score jumps) for variables with index in $\cC$. For example, $p_{S^{\hat{\cG}}_{\rm sum}}^{\cC}$ is the p-value for the statistic $\sum_{i\notin\cC}S(X_i\mid\PA_i^{\hat{\cG}})$ given by the (appropriately adjusted) equation in \cref{prop:fishertest}. We then have,
\begin{proposition}[{\bf Maximum leave-k-out p-value}]\label{prop:lkopvalue}
    For statistic T, $H_0$ can be rejected with p-value
    \begin{equation}
        p_T^{k} = \max_{\substack{\cC\subseteq\{1,\dots,n\} \\ \text{s.t. }|\cC|= k}} p_{T}^{\cC}.
    \end{equation}
\end{proposition}
Intuitively: if $p_T^k \leq\alpha$ then $p_{T}^{\cC}\leq\alpha$ for all $\cC$, including for all $\cC\supseteq\cR$. However, for $\cC\supseteq\cR$, $p_{T}^{\cC}$ is a valid p-value for $H_0$, so we may reject $H_0$ at level $\alpha$. \will{I am pretty sure this procedure is, in a certain sense, optimal among valid p-value combination techniques, suitable in this setting when, $|\cR| = k$. This is because if the top-$k$ highest scoring outliers correspond to the $k$ root causes, then $p_T^{k} = p_T$: i.e. the maximum leave-$k$-out p-value loses no power compared to the test when we know the root causes. However, we know from \citet{orchard2025root} that the maximum likelihood set of root causes, if it is known there are $k$ of them, is precisely the top-$k$ highest scoring outliers. So this procedure maximises the likelihood of losing no power.} 

\begin{figure*}[t]
    \centering
    \includegraphics[width=.7\textwidth]{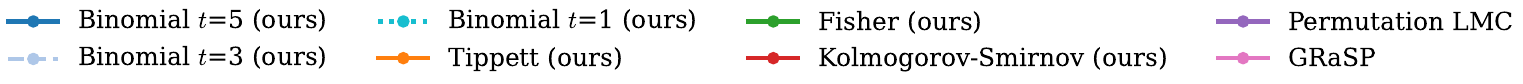}
    \begin{subfigure}[b]{\columnwidth}
        \centering
        \includegraphics[width=0.7\columnwidth]{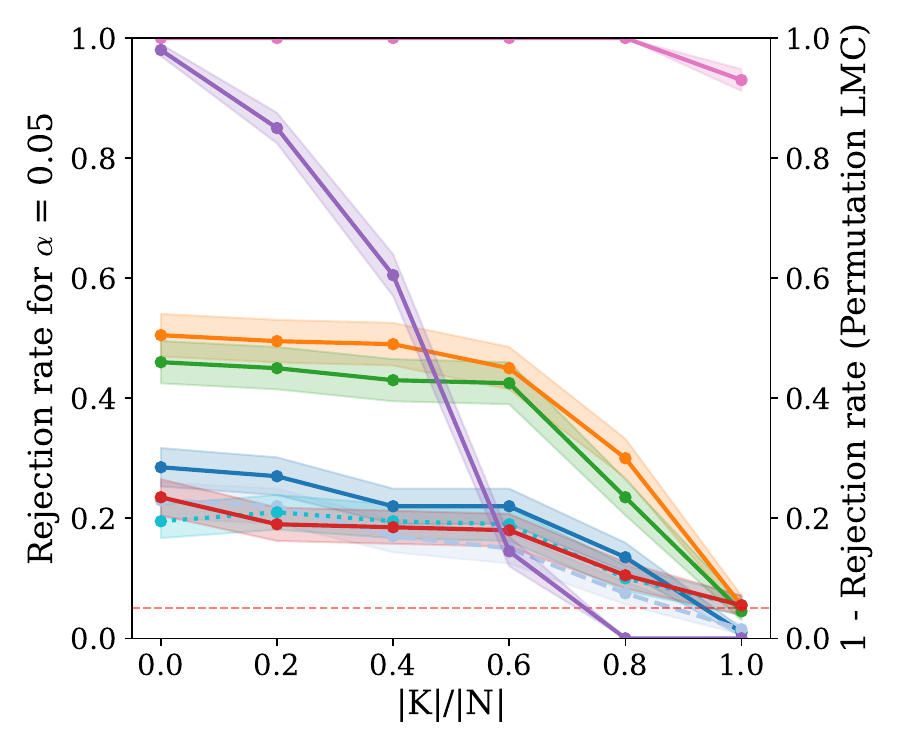}
        \caption{Conditional scores.}
        \label{fig:power_analysis_3_conditional}
    \end{subfigure}
    \begin{subfigure}[b]{\columnwidth}
        \centering
        \includegraphics[width=0.7\columnwidth]{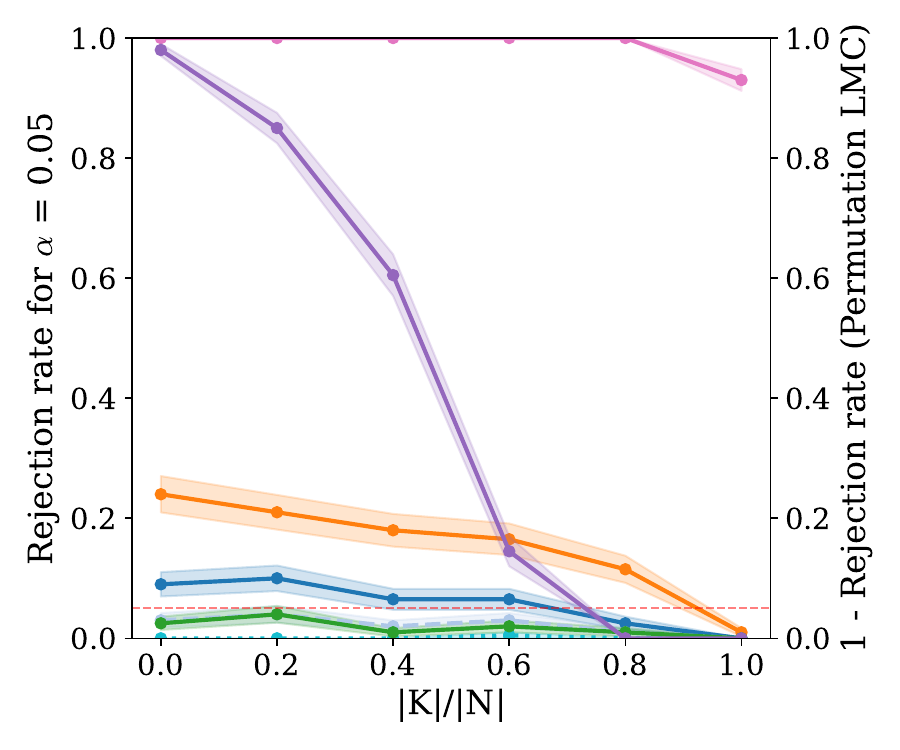}
        \caption{Marginal scores.}
        \label{fig:power_analysis_3_margina}
    \end{subfigure}
    
    \caption{Average rejection rate for synthetic data and increasing fraction of correctly connected nodes $|K|/|N|$.}
    \label{fig:power_analysis}
\end{figure*}

\section{Experiments}\label{sec:experiments}
In this section we compare our proposed tests to two baselines on both synthetic and real-world data. Our first baseline makes use of the GRaSP \citep{lam2022greedy} algorithm for causal discovery: we reject the given causal graph if it does not fall into the same Markov equivalence class as the graph empirically estimated by GRaSP. The second is the test proposed by \citet{eulig2025toward}. In brief, they test the hypothesis that the given graph is \emph{not} a better match of the conditional independencies inferred in the data than a permutation-based random baseline. As such, their test is rejected when the given graph is \emph{better than random}. Note that a graph may be better than random, but not the true causal graph, and so we expect their test to reject graphs more readily than ours do not reject them.

\looseness=-1For full experimental details, and further experiments, see \cref{sec:additional_experiments}. In brief, to generate synthetic data we sample a random DAG of 20 nodes, and assign to each node a randomly initiated linear model of its parents and a Gaussian noise variable, and draw 1000 samples. A root cause is chosen uniformly at random from the nodes and an outlier injected in its noise variable and propagated through the graph for a single additional sample. We adopt the approach of \citet{eulig2025toward} for sampling `expert-provided' graphs to falsify given the ground truth. In brief, a subset $K \subseteq N := \{X_1,\dots,X_n\}$ of nodes are selected at random. All edges between nodes in $K$ are preserved, while edges between the rest of the nodes are randomly resampled, so that a fraction $|K|/|N|$ of the graph is known.
\looseness=-1In \cref{fig:power_analysis}, we investigate how the rejection rate of our tests varies with $|K|/|N|$. 
As expected, our tests reject fewer graphs the higher $|K|/|N|$, with the first three of our tests having consistently higher rejection rates than the KS-test. Our marginal tests are less powerful than their conditional counterparts as expected. The GRaSP approach is far too sensitive to discrepancies and rejects almost all graphs at every $|K|/|N|$, including the ground truth. The permutation test displays the expected behaviour. In \Cref{sec:additional_experiments} we present several variations of this experiment, e.g., with larger injected outliers we see improvements in the power, and in the setting that the root causes are unknown we see only a slight drop in power.

    


\looseness=-1For the real-world data experiments we consider two datasets. First we used the PetShop dataset from \citet{hardt24a}. The dataset contains performance metrics for a microservice-based application before and after the introduction of outlier events. We test whether the so-called `service dependency map' of the application, with its edges reversed, is a good proxy for the causal graph of certain metrics, as has been suggested in the literature \citep{root_cause_analysis,li2022causal,Gan2021}. However, \citet{eulig2025toward}
has questioned this approach. Second we take the light-measurements from the Causal Chambers dataset \citep{gamella2025causal} for which the given graph is expected to be an accurate description of the causal structure. Results are displayed for both in \cref{fig:petshop_marginal}. We see the permutation test rejects the hypothesis that the graph is not better than random in all scenarios. In contrast, our tests reject the service graph in PetShop at a higher rate than under the null, especially in the high traffic setting, where the graph is rejected for almost all tested anomalies. This could suggest that while the service graph is better than random, it is nonetheless not the true causal graph. For Causal Chambers, none of our tests rejects at a higher rate than the significance level, suggesting the given causal graph is likely quite accurate. The GRaSP test always rejects the graph in every setting and for both datasets.

\begin{figure}[t]
    \centering
    \includegraphics[width=0.7\columnwidth, trim={0 1em 0 0},clip]{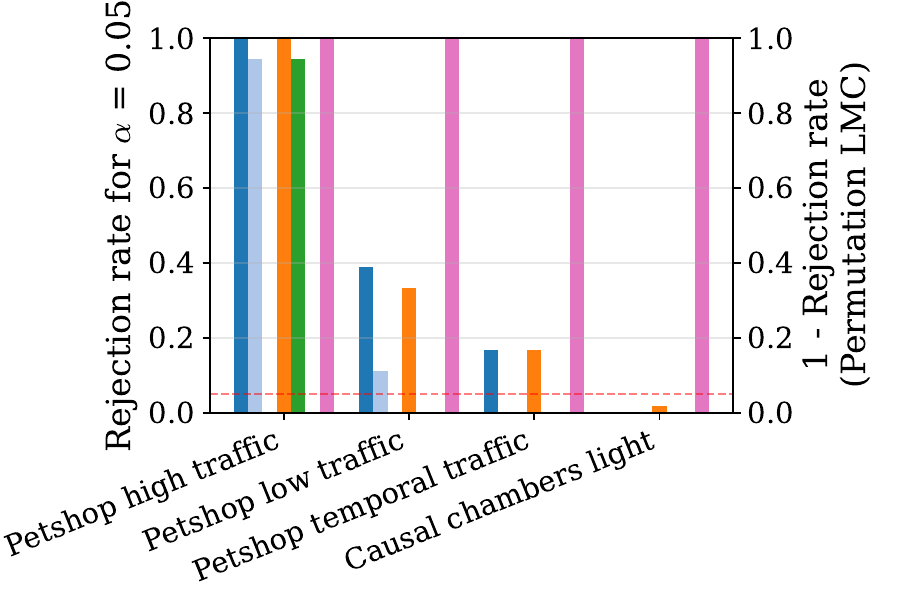}
    \caption{Average rejection rate of our marginal tests on cloud computing and Causal Chambers data.}
    \label{fig:petshop_marginal}
\end{figure}

{\bf Code:}\,\url{https://github.com/worchard/rca-falsification}
\section{Conclusion}\label{sec:conclusion}
\looseness=-1In this work, we demonstrate how outlier events can be used to falsify causal graphs.
Since rare events propagate along causal paths and rarely occur on their own, they are informative about causal structure, analogously to distribution shifts or interventions.
We have shown how to turn these theoretical insights from the root cause analysis literature into a statistical test that rejects false causal graphs.
We proved that this test has false positive control and power against certain classes of incorrect graphs even with only a single outlier sample.
We have experimentally validated these findings on synthetic and real-world data and have demonstrated that the power of our test is on par with other baselines while controlling for false positives.

\begin{acknowledgements} 
    W.R. Orchard would like to thank Cancer Research UK and Peterhouse, Cambridge for supporting this research. P. M. Faller was supported by a doctoral scholarship of the Studienstiftung des deutschen Volkes (German Academic Scholarship Foundation).
\end{acknowledgements}

\bibliography{bibliography}
\newpage

\onecolumn

\title{Falsifying Causal Graphs With Outlier Events\\(Supplementary Material)}
\maketitle
\begingroup
\renewcommand{\thefootnote}{*}
\footnotetext{Authors contributed equally.}
\endgroup

\appendix
\section{Proofs}\label{app:proofs}
\subsection{Proof of lemma~\ref{lem:scores_exponential_under_null}}
See proof of lemma 3.7 in \citet{orchard2025root}.

\subsection{Proof of proposition~\ref{prop:fishertest}}
By \cref{lem:scores_exponential_under_null}, under the null, $S^{\hat{\cG}}_{\rm sum}$ is the sum of $L$ independent, standard exponential random variables. The sum of $L$ independent, standard exponential random variables is Erlang distributed with CDF \citep{Ibe2013},
\begin{equation}
    F(x) = 1 - e^{-x}\sum_{l = 0}^{L - 1} \frac{x^l}{l!}.
\end{equation}
We therefore have,
\begin{equation}
    P(S^{\hat{\cG}}_{\rm sum} \geq s) = 1 - F(s) = e^{-s}\sum_{l = 0}^{L - 1} \frac{s^l}{l!},
\end{equation}
so that $1 - F(S^{\hat{\cG}}_{\rm sum})$ is a p-value for $H_0$.

\subsection{Proof of proposition~\ref{prop:tippetttest}}
Under the null $\{S(X_i\mid \PA_i^{\hat{\cG}})\}_{i\notin\cR}$ are independent, standard exponential so that, for each $i\notin\cR$, $P(S(X_i\mid \PA_i^{\hat{\cG}}) < s) = 1 - e^{-s}$, therefore,
\begin{equation}
    P(S_{\rm max}^{\hat{\cG}} < s) = P\left(S(X_i\mid\PA_i^{\hat{\cG}}) < s \:\:\text{for all }\: i\notin\cR \right) = \prod_{i\notin\cR} P(S(X_i\mid \PA_i^{\hat{\cG}}) < s) = (1 - e^{-s})^L.
\end{equation}
As such, $P(S_{\rm max}^{\hat{\cG}} \geq s) = 1 - (1 - e^{-s})^L$, and we have the result. 

\subsection{Proof of proposition \ref{prop:binomialtest}}
Proof given in main text immediately preceding the proposition.

\subsection{Proof of lemma \ref{lem:score_typical_pac}}
See proof of lemma 3.4 in \citet{orchard2025root}.

\subsection{Proof of proposition \ref{prop:marginaltests}}
Subject to score typicality we have the following,
\begin{equation}
    \hat{s}^{\hat{\cG}}_{\rm sum} := \sum_{i\notin\cR}|S(x_i) - S(\pa_i^{\hat{\cG}})|_+ \leq \sum_{i\notin\cR}S(x_i\mid\pa_i^{\hat{\cG}}) = s^{\cG}_{\rm sum},
\end{equation}
\begin{equation}
    \hat{s}_{\rm max}^{\hat{\cG}} := \max_{i\notin\cR}|S(x_i) - S(\pa_i^{\hat{\cG}})|_+ \leq \max_{i\notin\cR} S(x_i\mid\pa_i^{\hat{\cG}}) = s_{\rm max}^{\hat{\cG}},
\end{equation}
\begin{equation}
    \hat{n}_t := \sum_{i\notin\cR} \mathbb{I}[|S(x_i) - S(\pa_i^{\hat{\cG}})|_+ \geq t] \leq \sum_{i\notin\cR} \mathbb{I}[S(x_i\mid\pa_i^{\hat{\cG}}) \geq t] = n_t,
\end{equation}
so that in each case, the marginal score test statistic is stochastically dominated by the corresponding conditional score test statistic. Substituting the conditional score test statistics for each of the marginal ones, the corresponding outputs are therefore upper bounds for the p-values given in \crefrange{prop:fishertest}{prop:binomialtest}, and therefore are valid p-values for $H_0$.

\subsection{Proof of proposition \ref{prop:boundanm}}
We need first to introduce some extra concepts from information theory. In what follows, the exposition largely follows \citet{PolyanskiyWu2017}, however, the original references for the results quoted therein will be given here. First, contraction coefficients can be defined with respect to divergences other than the KL-divergence. For example, defining the total variation for distributions $P$ and $Q$ on the same measurable space, and events $E$ in the corresponding sigma-algebra,
\begin{equation}
    d_{\rm TV}(P, Q) := \sup_{E}|P(E) - Q(E)|,
\end{equation}
the contraction coefficient $\eta_{\rm TV}$ is given by simply replacing $D$ with $d_{\rm TV}$ in \cref{eq:contractioncoeff}. Second is that while in \cref{eq:contractioncoeff} we have given the so-called "input-dependent" contraction coefficient, more generally channel contraction coefficients are given without restriction to a certain input:
\begin{equation}
    \eta(P_{Y\mid X}) := \sup_{Q}\: \eta(P_{Y\mid X}, Q).
\end{equation}

With those additional preliminaries in place, we may proceed.
From theorem 1 in \citet{PolyanskiyWu2017} we have,
\begin{equation}
    \eta_{\rm KL}(P_{Y\mid X}) \leq \eta_{\rm TV}(P_{Y\mid X}),
\end{equation}
so that a general approach for upper bounding KL contraction coefficients is by inspecting their TV counterparts. To this end, \citet{Dobrushin1956} has shown that,
\begin{equation}
    \eta_{\rm TV}(P_{Y\mid X}) = \sup_{x, x'} d_{\rm TV}(P_{Y\mid X = x}, P_{Y\mid X = x'}).
\end{equation}
Noting that in our case, $P(Y \mid X = x) = \cN(f(x), \sigma^2)$, we use the fact that for $Q = \cN(\mu_1, \sigma^2)$ and $R = \cN(\mu_2, \sigma^2)$, $d_{\rm TV}(Q, R) = 2\Phi(|\mu_1 - \mu_2|/2\sigma) - 1$:
\begin{equation}
    \sup_{x, x'} d_{\rm TV}(P_{Y\mid X = x}, P_{Y\mid X = x'}) = \sup_{x, x'} \:2\Phi(|f(x) - f(x')|/2\sigma) - 1 \leq 2\Phi(B/\sigma) - 1,
\end{equation}
where we have used the fact that $f$ is bounded, and that $\Phi$ is strictly monotonically increasing. Finally, noting that $\eta_{\rm KL}(P_{Y\mid X}, Q) \leq \eta_{\rm KL}(P_{Y\mid X}) \leq \eta_{\rm TV}(P_{Y\mid X})$, we have the result.

\subsection{Proof of lemma \ref{lem:marginalscoreKL}}
Noting that for any event $E$ with $P_X(E) > 0$, we have $\frac{dP_{X\mid E}}{dP_X}(x) = \frac{\mathbf{1}_E(x)}{P_{X}(E)}$, $P_X\text{-a.e.}$. We then have from the definition of the KL-divergence (\cref{eq:kldivergence}) and the definition of marginal IT scores: 
\begin{equation}\label{eq:klscore}
    \KL{P_{X\mid E_x}}{P_X} = \int{\log\frac{\mathbf{1}_{E_x}(x')}{P_{X}(E_x)}dP_{X\mid E_x}(x')} = \log\frac{1}{P_X(E_x)} = S_X^{\tau}(x).
\end{equation}

\subsection{Proof of lemma \ref{lem:marginaldecay}}
From the SPDI in \cref{eq:sdpi} with inputs $P_{X\mid E_x}$ and $P_X$ and channel $P_{Y\mid X}$ we have
\begin{equation*}
    \eta\KL{P_{X\mid E_x}}{P_X} \geq \KL{\chan\circ P_{X\mid E_x}}{\chan\circ P_X},
\end{equation*}
so that, applying \cref{lem:marginalscoreKL} along with $P_Y = \chan\circ P_X$ and $P_{Y\mid E_x} = \chan\circ P_{X\mid E_x}$ we have,
\begin{equation}\label{eq:SxvsDy}
    \eta S_{X}^{\tau}(x)\geq \KL{P_{Y\mid E_x}}{P_Y}.
\end{equation}

For arbitrary distributions $P \ll Q$ and event $A$ such that $Q(A) > 0$, consider the binary channel defined by $Y = \mathbf{1}_A(X)$, then by the data processing inequality 
\begin{equation}
    \KL{P}{Q} \geq \KL{\text{Bern}(P(A))}{\text{Bern}(Q(A))} = P(A)\log{\frac{P(A)}{Q(A)}} + (1-P(A))\log{\frac{1-P(A)}{1-Q(A)}}.
\end{equation}
Taking the outlier event for $E_y := \{\tilde{\tau}(Y)\geq \tilde{\tau}(y)\}$, assume $P_Y(E_y) > 0$, we therefore have,
\begin{equation}\label{eq:binarybound}
\begin{aligned}
    &\KL{P_{Y\mid E_x}}{P_Y} \\
    &\geq P_{Y\mid E_x}(E_y)\log{\frac{P_{Y\mid E_x}(E_y)}{P_Y(E_y)}} + (1-P_{Y\mid E_x}(E_y))\log{\frac{1 - P_{Y\mid E_x}(E_y)}{1-P_Y(E_y)}} \\
    &= e^{-S_{Y\mid E_x}^{\tilde{\tau}}(y)}\log{\frac{e^{-S_{Y\mid E_x}^{\tilde{\tau}}(y)}}{e^{-S_{Y}^{\tilde{\tau}}(y)}}} + (1 - e^{-S_{Y\mid E_x}^{\tilde{\tau}}(y)})\log{\frac{1 - e^{-S_{Y\mid E_x}^{\tilde{\tau}}(y)}}{1 - e^{-S_{Y}^{\tilde{\tau}}(y)}}} \\
    &= e^{-S_{Y\mid E_x}^{\tilde{\tau}}(y)}(S_Y^{\tilde{\tau}}(y) - S_{Y\mid E_x}^{\tilde{\tau}}(y)) + (1 - e^{-S_{Y\mid E_x}^{\tilde{\tau}}(y)})\log{\frac{1 - e^{-S_{Y\mid E_x}^{\tilde{\tau}}(y)}}{1 - e^{-S_{Y}^{\tilde{\tau}}(y)}}}.
\end{aligned}
\end{equation}
Taking $S_{Y\mid E_x}^{\tilde{\tau}}(y)= 0$, the right hand side of \cref{eq:binarybound} is $ S_Y^{\tilde{\tau}}(y)$, and so finally, by combining \cref{eq:SxvsDy} with \cref{eq:binarybound}, we have
\begin{equation}\label{eq:SxvsSy}
    \eta S_X^{\tau}(x)\geq S_Y^{\tilde{\tau}}(y).
\end{equation}

\subsection{Proof of lemma \ref{lem:monotonimpliesbound}}
\begin{align}
    P(\tau(X)\geq \tau(x)\mid Y = y) &= P(S(X)\geq S(x)\mid S(Y) = S(y)) \\
    &\leq P(S(X)\geq S(x)\mid S(Y)\geq S(y))\\
    &\leq \frac{P(S(X)\geq S(x))}{P(S(Y)\geq S(y))},
\end{align}
where the equality follows from injectivity of $S_Y$ and the definition of IT outlier scores, the first inequality from the monotonicity property, and the final inequality because $P(B\mid A)\leq P(B)/P(A)$ for any two events $A, B$. Finally, taking the negative log-transform of each side, and using the definition of the conditional and marginal IT outlier scores, we reach the result.

\subsection{Proof of proposition \ref{prop:anmlogconcave}}
In order to understand this proof, we need to introduce some terminology from the theory of `total positivity' (see \citep{karlin1968total}). A function $h:\R\times\R\to\R$ is called totally positive of order two (TP$_2$) if for $x_1\leq x_2$ and $y_1\leq y_2$,
\begin{equation}\label{eq:tp2}
    h(x_2,y_1)h(x_1,y_2)\leq h(x_1,y_1)h(x_2,y_2).
\end{equation}
It is a fundamental result (see, for example, A.10 of \citep{Marshall2011}) that the function
\begin{equation}
    h(x, y) = g(y - x),
\end{equation}
is TP$_2$ if and only if $g$ is non-negative and log-concave on $\R$.

In our case, we have that
\begin{equation}
    p(y\mid x) = p_N(y - f(x)),
\end{equation}
where $p_N$ is log-concave by assumption. As such, \cref{eq:tp2} applies applies to $h(u, y) := p_N(y - u)$ for all $u_1, \leq u_2$ and $y_1 \leq y_2$. Furthermore, as $f$ is non-decreasing, we have $x_1\leq x_2\implies f(x_1)\leq f(x_2)$, and so taking $u_i := f(x_i)$, we have that $h(f(x), y) = p(y\mid x)$ is TP$_2$ as a function of $x$ and $y$. 

Substituting $p(y\mid x)$ into \cref{eq:tp2} and multiplying both sides by $p(x_1)p(x_2)$, we find that the joint density $p(x, y)$ satisfies the inequality too and is therefore TP$_2$. It is relatively straightforward to show (see theorem 5.2.19 of \citep{Nelsen2006} and note that they refer to TP$_2$ in this context as PLR) that if $p(x, y)$ is TP$_2$ then $P(X\geq x\mid Y = y)$ is a non-decreasing function of $y$ for all $x$. As such, by taking the expectation of $P(X\geq x\mid Y = y)$ with respect to $P(Y\mid Y\geq y)$ this implies
\begin{equation}\label{eq:stochasticincrease}
    P(X\geq x\mid Y = y)\leq P(X\geq x\mid Y\geq y).
\end{equation}

The final step is to note that, as both $S_X, S_Y$ are strictly increasing we have $S(X)\geq S(x) \iff X \geq x$ and $S(Y)\geq S(y) \iff Y\geq y$, so that \cref{eq:stochasticincrease} can be restated as precisely the monotonicity property of the proposition.

\subsection{Proof of lemma \ref{lem:boundconditionalnondesc}}
For clarity of exposition, denote $X_i$ as $Y$ and $\PA_i^{\hat{\cG}}$ as $X$. The probability of interest is $\tilde{P}(S(Y\mid X)\geq \lambda -\log c)$, for which we need a lower bound. First note,
\begin{align}
    \tilde{P}(S(Y\mid X)\geq \lambda -\log c) &\geq \tilde{P}(S(Y)\geq \lambda,S(Y\mid X)\geq \lambda -\log c) \\
    &= \tilde{P}(S(Y)\geq \lambda)\tilde{P}(S(Y\mid X)\geq \lambda - \log c\mid S(Y)\geq \lambda).
\end{align}
However, by assumption $Y$ is a descendant of a root cause in $\cG$, so we have, subject to effectiveness, that $\tilde{P}(S(Y)\geq \lambda) = 1$, and as such we only need to lower bound $\tilde{P}(S(Y\mid X)\geq \lambda - \log c\mid S(Y)\geq \lambda)$. Equivalently we need to upper bound $\tilde{P}(S(Y\mid X) < \lambda - \log c\mid S(Y)\geq \lambda)$. 

Note now that, conditional on $S(Y)\geq\lambda$, for any outlier sample $(x, y)$, clearly $S(y)\geq \lambda$, so that $P(S(Y)\geq S(y)\mid X = x)\leq P(S(Y)\geq\lambda\mid X = x)$, and by taking negative log-transforms that $S(y\mid x) \geq S(\lambda\mid x)$. Accordingly, for such samples, $S(y\mid x) < \lambda - \log c\implies S(\lambda\mid x) < \lambda -\log c,$ and therefore that,
\begin{align}
    \tilde{P}(S(Y\mid X) < \lambda -\log c \mid S(Y)\geq\lambda) &\leq \tilde{P}(S(\lambda\mid X) < \lambda -\log c \mid S(Y)\geq\lambda) \\
    &\leq \frac{\tilde{P}(S(\lambda\mid X) < \lambda -\log c)}{\tilde{P}(S(Y)\geq\lambda)} \\
    &= \tilde{P}(S(\lambda\mid X) < \lambda -\log c),
\end{align}
where the second inequality follows from the fact that $P(B\mid A)\leq P(B)/P(A)$ for any two events $A, B$, and the equality follows from effectiveness as before.

What is key is that this final expression now depends only on $X$, and by assumption $X$ is not a root cause or a descendant of a root cause so $\tilde{P}_X = P_X$ and therefore
\begin{equation}
    \tilde{P}(S(\lambda\mid X) < \lambda -\log c) = \tilde{P}_X(S(\lambda\mid X) < \lambda -\log c) = P_X(S(\lambda\mid X) < \lambda -\log c).
\end{equation}

Finally, noting that $\bE_{P_X}[P(S(Y)\geq\lambda\mid X)] = P(S(Y)\geq\lambda)\leq e^{-\lambda}$, we may apply Markov's inequality to conclude
\begin{align}
    P_X(S(\lambda\mid X) < \lambda -\log c) &= P_X(P(S(Y)\geq \lambda\mid X) > ce^{-\lambda}) \\
    &\leq \frac{e^{-\lambda}}{ce^{-\lambda}} = 1/c.
\end{align}
Putting everything together, we have the result.

\subsection{Proof of lemma \ref{lem:lowboundexpect}}
By Bayes theorem
\begin{equation}
    P_X(E_x\mid Y = y) = P_X(E_x)\frac{dP_{Y\mid E_x}}{dP_Y}(y),
\end{equation}
so that by taking a negative log-transform, and applying the definition of IT outlier scores,
\begin{equation}
    S(x\mid y) = S(x) - \log \frac{dP_{Y\mid E_x}}{dP_Y}(y).
\end{equation}
Next, taking the expectation with respect to $P_{Y\mid E_x}$, we find
\begin{equation}
    \bE_{P_{Y\mid E_x}}[S(x\mid Y)] = S(x) -\bE_{P_{Y\mid E_x}}\left[\log \frac{dP_{Y\mid E_x}}{dP_Y}(Y)\right].
\end{equation}
Note that the last term on the right is exactly $\KL{P_{Y\mid E_x}}{P_Y}$, so that by the SDPI and \cref{lem:marginalscoreKL}, we have
\begin{equation}
    \bE_{P_{Y\mid E_x}}[S(x\mid Y)] \geq S(x) - \eta S(x) = (1-\eta)\,S(x).
\end{equation}

\subsection{Proof of lemma \ref{lem:conditionallowbound}}
For real-valued random variable $X$ with finite variance, Cantelli's inequality (sometimes called the one-sided Chebyshev inequality) \citep{Ghosh2002} is 
\begin{equation}
    P\left(X \geq \bE[X] -\sqrt{\frac{1-\delta}{\delta}{\rm Var[X]}}\right) \geq 1- \delta,
\end{equation}
so that, in our case,
\begin{equation}
      P_{Y\mid E_x}\left(S(x\mid Y) \geq \bE[S(x\mid Y)] -\sqrt{\frac{1-\delta}{\delta}V_x}\right) \geq 1- \delta.
\end{equation}
We then may substitute $\bE_{P_{Y\mid E_x}}[S(x\mid Y)]$ for its lower bound in \cref{lem:lowboundexpect} and $V_x$ for any finite upper bound.

\subsection{Proof of proposition \ref{prop:anmboundedvariance}}
By Bayes theorem
\begin{equation}
    P_X(E_x\mid Y = y) = P_X(E_X)\frac{p_{Y\mid E_x}(y)}{p_Y(y)},
\end{equation}
so that by taking a negative log-transform, and applying the definition of IT outlier scores,
\begin{equation}\label{eq:itscorebayes}
    S(x\mid y) = S(x) -\log\frac{p_{Y\mid E_x}(y)}{p_Y(y)}.
\end{equation}
Note that only the last term on the right depends on $y$, so that it is sufficient to bound it in order to bound ${\rm Var}_{P_{Y\mid E_x}}(S(x\mid Y))$ in $y$ for each $x$. Noting first that $p(y\mid x) = p_N(y - f(x))$, let's first consider the log-Lipschitz case.

As $p_N$ is log-Lipschitz we have that, for any $a, b$
\begin{align}
    |\log p_N(y - f(a)) - \log p_N(y - f(b))| &\leq L_N |(y - f(a)) - (y - f(b))| \\
    &= L_N |f(a) - f(b)| \\
    &\leq 2L_NB,
\end{align}
where the final inequality is due to $f$ being bounded. As such we have
\begin{equation}\label{eq:lipschitzstep1}
    e^{-2L_NB} \leq \frac{p_N(y - f(a))}{p_N(y - f(b))} \leq e^{2L_NB},
\end{equation}
so that
\begin{equation}\label{eq:lipschitzstep2}
    e^{-2L_NB}p_N(y - f(b)) \leq p_N(y - f(a)) \leq e^{2L_NB}p_N(y - f(b)).
\end{equation}
Averaging first over $a$ with respect to $P_{X\mid E_x}$, and then $b$ with respect to $P_X$ we have
\begin{equation}\label{eq:lipschitzstep3}
    e^{-2L_NB}p_Y(y) \leq p_{Y\mid E_x}(y) \leq e^{2L_NB}p_Y(y),
\end{equation}
so that together, dividing by $p_Y(y)$ and taking log-transforms we have
\begin{equation}\label{eq:lipschitzstep4}
    \left|\log\frac{p_{Y\mid E_x}(y)}{p_Y(y)} \right| \leq 2L_NB.
\end{equation}
Returning to \cref{eq:itscorebayes}, this then immediately implies
\begin{equation}\label{ineq:loglipschitzscorebound}
    S(x) - 2L_NB \leq S(x\mid y) \leq S(x) + 2L_NB,
\end{equation}
holds for all $x$ and $y$. Equivalently, we have $|S(x\mid y) - S(x)| \leq 2L_NB$. Finally, we have
\begin{align}
    {\rm Var}_{P_{Y\mid E_x}}[S(x\mid Y)] &= {\rm Var}_{P_{Y\mid E_x}}[S(x\mid Y) - S(x)] \\
    &= \bE_{P_{Y\mid E_x}}[(S(x\mid Y) - S(x))^2] - \bE_{P_{Y\mid E_x}}[S(x\mid Y) - S(x)]^2 \\
    &\leq \bE_{P_{Y\mid E_x}}[(S(x\mid Y) - S(x))^2] \\
    &\leq (2L_NB)^2.
\end{align}
This completes the proof for part one of the proposition. Note in addition, however, that if we take one side of \cref{ineq:loglipschitzscorebound}, we simply have
\begin{equation}
    S(x\mid y) \geq S(x) - 2L_NB,
\end{equation}
and this bound holds uniformly, for all $x$ and $y$, and is often stronger than that given by Cantelli's inequality in \cref{lem:conditionallowbound}.

We now turn our attention to part two of the proposition where $N\sim \cN(0, \sigma^2)$. The proof follows a similar technique to part one, but we cannot now rely on log-Lipschitz continuity. First note, for any $a$ and $b$,
\begin{align}
    |\log p_N(y -f(a)) -\log p_N(y - f(b))| &= \frac{1}{2\sigma^2}\left|(y - f(b))^2 - (y - f(a))^2\right| \\
    &=\frac{1}{2\sigma^2}\left|2y(f(a) - f(b)) +f(b)^2 - f(a)^2 \right| \\
    &\leq \frac{1}{2\sigma^2}\left(4|y|B + B^2\right) \\ &= \frac{2B}{\sigma^2}|y| + \frac{B^2}{2\sigma^2},
\end{align}
where the first equality follows from $p_N$ being the Gaussian density, and the inequality from $|f(a) - f(b)| \leq 2B$ and $|f(b)^2 - f(a)^2| \leq B^2$, both of which are due to $f$ being bounded.

It follows from exactly the same steps as we followed in \cref{eq:lipschitzstep1,eq:lipschitzstep2,eq:lipschitzstep3,eq:lipschitzstep4} that
\begin{equation}
    \left|\log\frac{p_{Y\mid E_x}(y)}{p_Y(y)} \right| \leq \frac{2B}{\sigma^2}|y| + \frac{B^2}{2\sigma^2}.
\end{equation}
For ease of exposition, let us rewrite the bound as
\begin{equation}
    C(y) := \alpha |y| + \beta,
\end{equation}
where $\alpha := 2B/\sigma^2$ and $\beta := B^2/2\sigma^2$, and noting that, as before, we have ${\rm Var}_{P_{Y\mid E_x}}[S(x\mid Y)] \leq \bE_{P_{Y\mid E_x}}[C(Y)^2]$. As such, note
\begin{equation}
    \bE_{P_{Y\mid E_x}}[C(Y)^2] = \alpha^2\bE_{P_{Y\mid E_x}}[Y^2] + 2\alpha\beta\bE_{P_{Y\mid E_x}}[|Y|] + \beta^2.
\end{equation}
To give the final expression, we need to evaluate the two expectations. Noting that for $Y\sim P_{Y\mid E_x}$, we have $Y = f(X) + N$ where $X\sim P_{X\mid E_x}$ and $X\ind N$, we have
\begin{align}
    \bE[Y^2] &= \bE_[(f(X) + N)^2] \\
    &= \bE[f(X)^2 + 2f(X)N + N^2] \\
    &= \bE[f(X)^2] + \bE_{P_{Y\mid E_x}}[N^2] \\
    &\leq B^2 + \sigma^2
\end{align}
where in the third equality we used the fact that $\bE[2f(X)N] = 2\bE[f(X)]\bE[N] = 0$ as $\bE[N] = 0$, and in the inequality that $f$ is bounded.

By Cauchy-Schwarz we additionally have $\bE[|Y|]\leq \sqrt{\bE[Y^2]} \leq \sqrt{B^2 + \sigma^2}$. Finally, by substituting both bounds on the expectations and the definitions of $\alpha$ and $\beta$, we have the result.

\subsection{Proof of proposition \ref{prop:pvaluebounds}}
All three bounds follow from substituting lower bounds on each of the corresponding test statistics into each of their p-value functions. In the `worst-case', all the conditional scores (or marginal score jumps) from among the $\pi L$ non-root causes that are subject to errors in $\hat{\cG}$ are equal to $b$, and all of the remaining $(1-\pi)L$ conditional scores (or marginal score jumps) are $0$. In such a case we have, in the case for marginal score jumps,
\begin{align}
     \hat{s}^{\hat{\cG}}_{\rm sum} &= \sum_{i\notin\cR}|S(x) - S(\pa_i^{\hat{\cG}})|_+ = \pi L b, \\
     \hat{s}^{\hat{\cG}}_{\rm max} &= \max_{i\notin\cR}|S(x) - S(\pa_i^{\hat{\cG}})|_+ = b, \\
     \hat{n}_t &= \sum_{i\notin\cR} \mathbb{I}[S(x_i\mid\pa_i^{\hat{\cG}}) \geq t] = \begin{cases}
        \pi L & \text{if } b \ge t, \\
        0 & \text{otherwise}.
    \end{cases}
\end{align}
With the same bounds applying to the conditional score test p-values with the same reasoning (though under different assumptions).

\subsection{Proof of theorem \ref{thm:kstestpower}}
The following is subject to alternative hypothesis $H_1$. For convenience, denote the CDF of the standard exponential distribution $F_{\rm Exp}(s) = (1 - e^{-s})$
\begin{align}
    \nonumber &\sup_{s}|F(s) - F_{\rm Exp}(s)| \geq \sup_{s<b}|F(s) - F_{\rm Exp}(s)| \\\nonumber &=\sup_{s<b}|(1-\pi)F_{\rm Exp}(s) - F_{\rm Exp}(s)| =\sup_{s<b}|\pi F_{\rm Exp}(s)| \\
    \label{eq:popDbound} &= \pi F_{\rm Exp}(b).
\end{align}
Let $F_L$ be the empirical distribution of $F$ for $L$ variables, our secondary test statistic is the Kolmogorov-Smirnov statistic $D_L$ for $L$ samples. We then have,
\begin{align}
    \nonumber D_L &:= \sup_{s}|F_L(s) - F_{\rm Exp}(s)| \\
    \nonumber &= \sup_{s}|F_L(s) - F(s) + F(s) - F_{\rm Exp}(s)| \\
    \nonumber &\geq \sup_{s}|F(s) - F_{\rm Exp}(s)| - \sup_{s}|F_L(s) - F(s)| \\
    \label{eq:sampleDbound} &\geq \pi F_{\rm Exp}(b) - \sup_{s}|F_L(s) - F(s)|,
\end{align}
where the first inequality follows from the triangle inequality, 
and the second from Equation~\ref{eq:popDbound}. For $\pi > 0$, and significant threshold $\alpha$, the power of the test is $P(D_L\geq D^{\rm crit}_{L,\alpha})$. If $\pi F_{\rm Exp}(b) > D^{\rm crit}_{L,\alpha}$, then for all $\varepsilon > 0$ such that $\pi F_{\rm Exp}(b) - \varepsilon \geq D^{\rm crit}_{L,\alpha}$, we have that $\sup_{s}|F_L(s) - F(s)| \leq \varepsilon$ implies $D_L\geq D^{\rm crit}_{L,\alpha}$ by Equation~\ref{eq:sampleDbound}. Therefore, in such cases,
\begin{equation}\label{eq:explicitpowerbound}
    P(D_L\geq D^{\rm crit}_{L,\alpha}) \geq P(\sup_{s}|F_L(s) - F(s)| \leq \varepsilon).
\end{equation}
Finally, by the Dvoretzky–Kiefer–Wolfowitz inequality \citep{DKW1956, Massart1990}, for all $\varepsilon > 0$,
\begin{equation}\label{eq:dkw}
    P(\sup_{s}|F_L(s) - F(s)| \leq \varepsilon) \geq 1 - 2e^{-2L\varepsilon^2},
\end{equation}
so that picking $\varepsilon = |\pi F_{\rm Exp}(b) - D^{\rm crit}_{L,\alpha}|_+$, and combining Equation~\ref{eq:explicitpowerbound} with Equation~\ref{eq:dkw}, we have the result. Note that the theorem trivially holds for $\pi F_{\rm Exp}(b) \leq D^{\rm crit}_{L,\alpha}$ too.

\subsection{Proof of theorem \ref{thm:binomialoptimal}}
Taking the likelihood ratio for a single variable we have,
\begin{equation}
    \frac{\pi e^{-(S(x_i\mid\pa_i^{\hat{\cG}})-t)}\mathbb{I}[S(x_i\mid\pa_i^{\hat{\cG}})\geq t] + (1 - \pi) e^{-S(x_i\mid\pa_i^{\hat{\cG}})}}{e^{-S(x_i\mid\pa_i^{\hat{\cG}})}} = \pi e^t \mathbb{I}[S(x_i\mid\pa_i^{\hat{\cG}})\geq t] + (1-\pi)
\end{equation}

So for the full sample, the ratio is
\begin{align}
    \Lambda_b &= \prod_{i \notin \cR}\left(\pi e^t \mathbb{I}[S(x_i\mid\pa_i^{\hat{\cG}})\geq t] + (1-\pi)\right)\\
    &= \left(\pi e^t + (1-\pi)\right)^{N_t}(1-\pi)^{L - N_t},
\end{align}
Note that, for fixed $t$, and for every $\pi$, likelihood ratio is monotonically increasing with increasing $N_t$ so that a test based on the upper tail of the ratio is equivalent to one on the upper tail of $N_t$. By the Karlin-Rubin theorem \citep{KarlinRubin1956}, for fixed $t$, \cref{prop:binomialtest} defines a test which is uniformly most powerful against the alternative $H^{\pi, t}_1$.

\subsection{Proof of proposition \ref{prop:lkopvalue}}
Let $H_0^{\cC}$ be the null hypothesis that all $x_i$ with $i\notin\cC$ were drawn according to $P(X_i\mid\PA_i^{\hat{\cG}} = \pa_i^{\hat{\cG}})$. Since $|\cR|\leq k$, there exists a set
$\cC^*\subseteq \{1,\dots,n\}$ with $|\cC^*|=k$ and
$\cR\subseteq \cC^*$, which is included in the maximisation defining $p_T^k$. As such, note that $H_0$ implies $H_0^{\cC^*}$: if $H_0$ is true, then after excluding the variables in $\cC^*$, all remaining variables are non-root causes, and so are all drawn according to their true causal mechanisms as the causal conditionals implied by $\hat{\cG}$ coincide with those in $\cG$. Rejecting $H_0^{\cC^*}$ is therefore sufficient to reject $H_0$. As $p_T^{\cC^*}$ is a p-value for $H_0^{\cC^*}$, it is therefore a p-value for $H_0$. Finally note,
\begin{equation}
    P_{H_0}\left(\max_{\substack{\cC\subseteq\{1,\dots,n\} \\ \text{s.t. }|\cC|= k}} p_{T}^{\cC} \leq \alpha\right) = P_{H_0}\left(\bigcap_{\cC}\{p_T^{\cC} \leq \alpha\}\right) \leq P_{H_0}(p_T^{\cC^*} \leq \alpha) \leq\alpha,
\end{equation}
so $p_T^k$ is a p-value for $H_0$.

\section{Additional marginal outlier decay result}
\label{app:marginaldecay}

\Cref{lem:marginaldecay} holds exactly for $S(y\mid E_x) = 0$, however, more generally we have,
\begin{proposition}
    For any $x$ such that $P(E_x) > 0$, let $y$ be chosen at random according to $P(Y\mid E_x)$. With probability at least $1 - \delta$, we obtain a $y$ for which,
    \begin{equation}
        \eta S(x) \geq \delta S(y) -\log 2.
    \end{equation}
\end{proposition}
\begin{proof}
    Denote for convenience $p(y) := P_{Y\mid E_x}(E_y)$ and $q(y) = P_Y(E_y)$. We may conclude from \cref{eq:binarybound} in the proof of \cref{lem:marginaldecay} that
    \begin{align}
        \eta S(x) &\geq p(y)\log\frac{p(y)}{q(y)} + (1-p(y))\log\frac{1 - p(y)}{1 - q(y)} \\
        &\geq p(y)\log\frac{p(y)}{q(y)} + (1-p(y))\log(1 - p(y)) \\
        &=p(y)\log p(y) + (1-p(y))\log(1 - p(y)) - p(y)\log q(y) \\
        &= -H_b(p(y)) -p(y)\log q(y),
    \end{align}
    where the second inequality follows from $-\log(1-q(y)) \geq 0$, and $H_b$ denotes the binary entropy function. Noting that $\max_{u\in[0,1]} H_b(u) = \log 2$, and rewriting $p(y)$ and $q(y)$ we therefore have,
    \begin{equation}
        \eta S(x) \geq e^{-S(y | E_x)}S(y) -\log 2.
    \end{equation}
    By the properties of IT scores $P_{Y\mid E_x}(S(Y | E_x) \leq c) = 1 - e^{-c}$, so that
    \begin{equation}
        P_{Y\mid E_x}\left(\eta S(x) \geq e^{-c}S(Y) -\log 2\right) \geq 1 - e^{-c}.
    \end{equation}
    The result is then given by defining $\delta := e^{-c}$.
    
\end{proof}
\section{Upper bounds on p-values}
\label{sec:power_curves}
In \cref{fig:p-value_bounds} we can see the upper bounds on the p-values of our tests from \cref{prop:pvaluebounds}. 
The bound for $p_{S_{\max}^\cG}$ is constant with respect to $\pi$, and so can be strongly preferable to the other tests when $\pi$ is small. In contrast $p_{S_{\text{sum}}^\cG}$ and $p_{N_t}$ are often clearly preferable for relatively low $b$ as $\pi$ increases. Notably, the tail decays much more sharply with increasing $\pi$ for $p_{N_t}$ compared to $p_{S_{\text{sum}}^\cG}$ for suitable settings of $t$, but it is insensitive, beyond the threshold, to increases in $b$.

These upper bounds can also help to understand how the power of our tests depend on the size of the graph being tested. Supposing the number of root causes is fixed, then the size of the graph is characterised by $L$. If we assume $\pi$ and $b$ are fixed then as $L$ increases, the upper bound on the p-value for the Tippett’s method test tends to $1$. For the Fisher’s method test and the Binomial test, whether the upper bounds on their p-values tends to $0$, tends to 
$1/2$ or tends to $1$ depends on the relative values of $\pi$ and $b$ in the case of the former, and on the relative values of $\pi$ and $t$ for the latter. If we also consider the lower bound on the power for the Kolmogorov-Smirnov test given in \cref{thm:kstestpower}, we see it tends to $1$ as $L$ grows. As such, the Kolmogorov-Smirnov test becomes more powerful for larger graphs, the Tippett’s method test suffers, and the remaining tests may become more powerful, less powerful or remain stable depending on the other factors. In practice, we might expect both the number of root causes, the strength of the outlier among descendants (quantified by $b$) and the proportion of errors to vary with graph size, and so to give precise characterisation on power in this setting would require making modelling assumptions concerning how and when these parameters vary for larger $L$.

\begin{figure*}[b]
    \centering
    \begin{subfigure}[b]{0.48\textwidth}
        \centering
        \includegraphics[width=\textwidth]{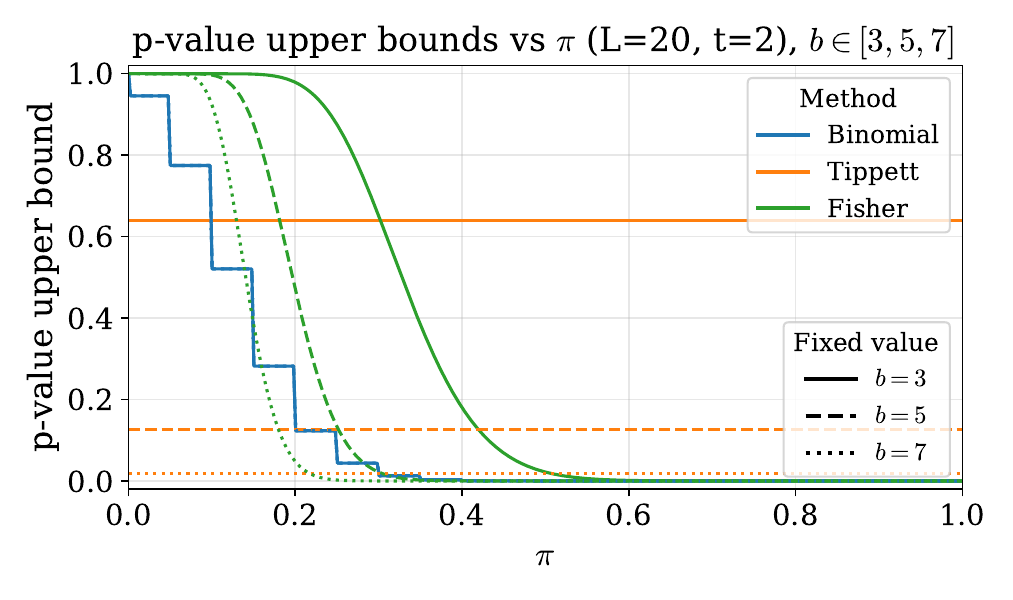}
        \caption{}
        \label{fig:p_value_bounds_first}
    \end{subfigure}
    \hfill
    \begin{subfigure}[b]{0.48\textwidth}
        \centering
        \includegraphics[width=\textwidth]{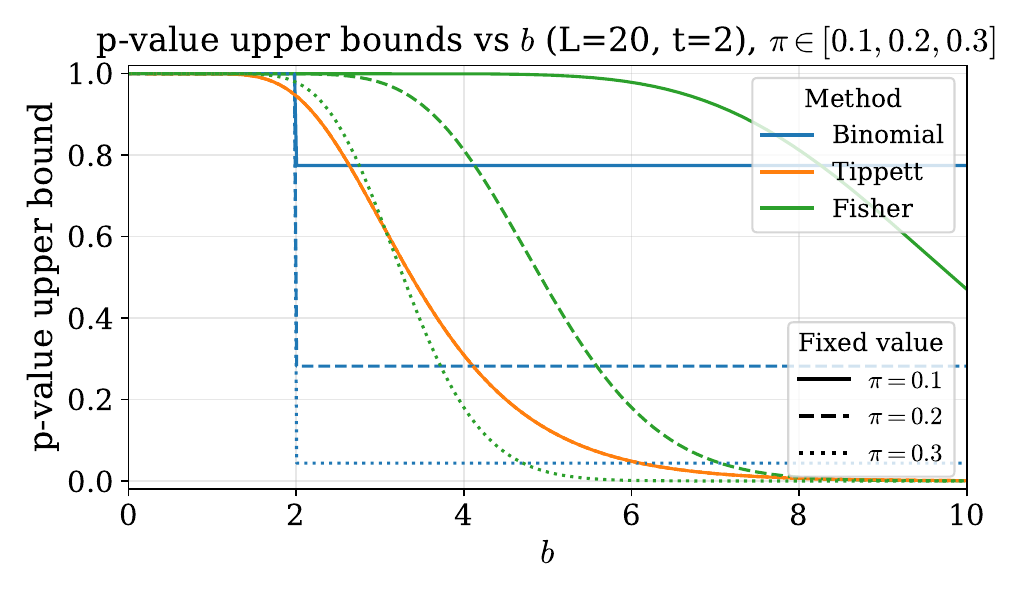}
        \caption{}
        \label{fig:p_value_bound_second}
    \end{subfigure}
    \caption{Bounds on the p-values of our tests.}
    \label{fig:p-value_bounds}
\end{figure*}

\section{When is each test uniformly most powerful}\label{app:ump}
\will{probably need to clear this up}
As the first test is simply an IT score analogue to Fisher's method, as mentioned, we can adapt a result from \citet{HeardRubinDelanchy2018},
\begin{theorem}[{\bf Optimality of sum of conditional outlier scores test}]\label{thm:optimalfisher}
    The test defined by $p_{S^{\hat{\cG}}_{\rm sum}}$ in \cref{prop:fishertest} is uniformly most powerful against the alternative
    \begin{hypothesis}{$H_1$}
    $S(X_i\mid\PA_i^{\hat{\cG}})$, for $i\notin\cR$, are iid exponential random variables with rate parameter $\nu < 1$.
    \end{hypothesis}
\end{theorem}
\begin{proof}
    Immediate consequence of \cref{lem:scores_exponential_under_null} and example 1 in \citep{HeardRubinDelanchy2018}.
\end{proof}
Similarly, for the second test,
\begin{theorem}[{\bf Optimality of the max conditional outlier score test}]\label{thm:optimaltippett}
    The test defined by $p_{S_{\rm max}^{\hat{\cG}}}$ in \cref{prop:tippetttest} is uniformly most powerful against the alternative $H_1:$
    \begin{align*}
        S(X_j\mid\PA_j^{\hat{\cG}})&\sim \frac{1}{B(\alpha, L)} e^{-\alpha s}(1 - e^{-s})^{L -1}, \quad
        j\notin\cR, \alpha\in(0,1), \\
        S(X_i\mid\PA_i^{\hat{\cG}})&\sim \frac{e^{-s}}{1 - e^{-S(x_j\mid\pa_j^{\hat{\cG}})}}, \quad
        i\neq j \notin\cR, s\in [0, S(x_j\mid\pa_j^{\hat{\cG}})],
    \end{align*}
    where $\alpha$ is a shape parameter of the Beta distribution, and $B$ is the beta function.
\end{theorem}
\begin{proof}
    Consequence of our \cref{lem:scores_exponential_under_null} and proposition 2 in \citep{HeardRubinDelanchy2018} after a change of variables.
\end{proof}

\section{Further comments on related work}
In \cref{sec:related_work} we cite a number of related works that may initially appear to be suitable candidates for comparison with our tests in \cref{sec:experiments}, and so the question may arise as to why we do not include them. We address first the goodness-of-fit tests developed by \citet{canonne2017testing,schkoda2025goodness}. In short, it is not clear how such a comparison would be done. The approach of \citet{schkoda2025goodness} tests the hypothesis that some (observational) data is consistent with \emph{any} linear structural equation model, and so does not allow for testing a given or estimated causal graph, as the null hypothesis is agnostic as to the underlying causal graph. \citep{canonne2017testing} is more complex. Firstly, they restrict themselves to binary variable Bayesian networks in the observational setting (whereas we are concerned with the continuous setting with outlier(s)). Secondly, they do not propose a test in the classical hypothesis testing sense, but rather a `testing algorithm' which rejects if the data distribution is at least a user-specified $\varepsilon$ away (in $L_1$ distance) from any binary Bayesian network with the given graph, with some probability. As such, it is not clear what a suitable $\varepsilon$ should be if we were to use their approach as the basis of comparison. Accordingly, it is not clear how we would perform comparisons with their approach either.

Similar comments can be made regarding the approaches of \citet{faller2024self} and \citet{schkoda2024cross}. Both approaches falsify, in the observational setting, causal discovery algorithms (relative to a given observational dataset) rather than falsifying a given causal graph, and so a what would constitute an appropriate direct comparison with out tests is not clear.

Another work that is relevant to discuss is that of \citet{Gnecco2021} as, at a high-level their work also uses the properties of tail events to infer causal directions (this work for falsification, theirs for causal discovery). However, there are key differences. Firstly, their model is parametric. Secondly, for them, extreme events are expected since the distribution is long-tailed, while we model extreme events as observations that are not drawn from the hypothetical distribution. This second difference is actually just a change of perspective, after all, we could also consider a mixture of the distribution considered in the null hypothesis with one that describes the outlier. The first difference, however, is more significant because their model predicts that, if $X\to Y$, then extreme values in $X$ are \emph{likely} to cause extreme values in $Y$, while this need not be the case in our non-parametric setting.

\section{Additional experiments and further details}
\label{sec:additional_experiments}
\subsection{Experimental Details}
For the experiment in \cref{fig:power_analysis,fig:power_analysis_sigma_6} we generated 200 DAGs randomly according to an Erd\H{o}s-R\'{e}nyi model with 20 nodes and three expected edges per node. 
We sample coefficients for linear structural equations uniformly from $[-1, 3]$\footnote{We observed empirically that an asymmetric interval allows the anomalies to propagate further through the graphs.}.
We then add standard Gaussian noise, while root nodes are with equal probability either Gaussian or uniform noise, or are a mixture of two to four Gaussian distributions.
The means are picked uniformly from $[-1, 1]$ and also the number of components is picked uniformly. 
For each dataset, we generate 1000 samples. 
We then uniformly pick a node as root cause and introduce an anomaly, by increasing (or decreasing) the noise value by three standard deviations, with probability 1/2, respectively for \cref{fig:power_analysis} and six standard deviations for \cref{fig:power_analysis_sigma_6}.
We reject an anomalous sample if no node (not even the root cause) has a marginal anomaly score of at least 3. 
Unless stated otherwise, we always assume to only have access to a single anomalous sample.

In order to create the node-expert graphs with knowledge fraction $|K|/|N|$, where $K\subseteq N:=\{X_1, \dots, X_n\}$ is the set of nodes where the expert is correct, we start by uniformly at random picking these $|K|$ nodes.
We then pick a causal ordering uniformly at random, such that the relative ordering between the nodes in $K$ is preserved.
We keep the edges between nodes in $K$.
We then calculate the size of the graph theoretic cut of $K$, i.e. $c:= |\{(X_i, X_j) \in \cG \mid X_i\not\in K \lor X_j\not\in K\}|$.
Finally, we uniformly sample $c$ new edges from the set of possible edges (between nodes in $K$ and $N\setminus K$) that respect our new causal order.

In order to estimate the marginal outlier scores, we leverage methods from the Python package \texttt{DoWhy GCM} \citep{bloebaum2022dowhygcm}.
Specifically, we use the rescaled median CDF quantile scorer.
This empirically estimates the score
\begin{displaymath}
    S(X_i) = -\log\left(2  \min\left\{P(X_i > x_i) + \frac{P(X_i = x_i)}{2}, P(X_i < x_i) + \frac{P(X_i = x_i)}{2}\right\}\right),
\end{displaymath}
by counting samples that are larger or smaller than $x_i$ in the given non-anomalous dataset, respectively, and taking two times the minimum of these numbers.
The score is then rescaled with the negative natural logarithm. It is worth noting that while our theoretical results hold for any real-valued $\tau$ (so long as we satisfy continuity), in practice one should choose a feature map suited to their particular application, as it must be large when the corresponding variable should be considered `extreme' in that application. We do not opt for the median rescaled CDF quantile scorer for any special reason other than it is a broadly applicable, non-parametric anomaly scorer: the $z$-score would have also been suitable, but in general we leave this up to the user.

For the conditional scores, we first estimate causal mechanisms using the \texttt{auto} module from \texttt{DoWhy}.
This means, for each node $X_i$, we fit an additive noise model $f_i$ for the conditional distribution using several linear and non-linear estimators from the \texttt{scikit-learn} package \citep{scikit-learn}, where we use the \texttt{better} parameter from \texttt{DoWhy}.
The best model is picked via five-fold cross-validation.
Then synthetic samples from the conditional of $X_i$ given $\pa_i$ can be generated by predicting $f(\pa_i)$ and drawing noise uniformly from the empirical residuals $x_i^j - f_i(\pa_i^j)$, where $x^j_i$ and $\pa_i^j$ denote the values of $X_i$ and its parents in the given non-anomalous samples for $j=1, \dots, m\in\N$.
Finally, the conditional anomaly scores $S(x_i\mid \pa_i)$ are estimated using the rescaled median CDF quantile scorer on a synthetic sample of the conditional with 10000 samples. 

For the test proposed by \citet{eulig2025toward}, we use the implementation from \texttt{DoWhy}, but with the Fisher $Z$ test, as implemented in the \texttt{causal-learn} package \citep{zheng2024causal}.
We use 200 permutations and set the significance threshold of the independence tests to $\alpha=0.05$.
For GRaSP, we also use an implementation from \texttt{causal-learn} and use default parameters.
Since GRaSP only outputs a CPDAG, representing the Markov-equivalence class of the ground truth, we cannot calculate the structural Hamming distance directly. 

Concerning the datasets in \cref{fig:petshop_marginal,fig:petshop_conditional}, we only focus on the latency data from PetShop, drop metrics with fewer than ten unique values and use mean-imputation for missing values.
The dataset provides three different traffic scenarios and we consider 18 anomaly events.
As preprocessing, we take a single sample and conduct our test with this sample and the provided non-anomalous data for each event.
\cref{fig:petshop_marginal} shows the average rejection rate over these runs with the marginal version of our test and the baselines.
For Causal Chambers, we use the \texttt{lt\_interventions\_standard\_v1} dataset provided on the website.
We drop constant columns and intervention indicators from the dataset and ground truth graph, and also take a single sample from one of the 58 provided interventional scenarios.

Each experiment can be run within 48 hours on an EC2 instance of type \texttt{c5.4xlarge}.

\subsection{Histograms of p-values}
In \cref{fig:histogram_3} we show the p-values of all tests from \cref{fig:power_analysis} for $|K|/|N|=1$, i.e. under the null hypothesis.
We can see that the p-values of our conditional tests are roughly uniformly distributed, except for the Binomial test for larger values of $t$ where the tests are conservative even for the conditional score tests. The marginal score tests are, as expected, consistently more conservative than their conditional score counterparts (and indeed seem to be super-uniformly distributed).

\begin{figure*}[t]
\includegraphics[width=.7\textwidth]{img/linear_3/legend_perturbed_dag.pdf}
    \centering
    \begin{subfigure}[b]{\textwidth}
        \centering
        \includegraphics[width=.8\textwidth]{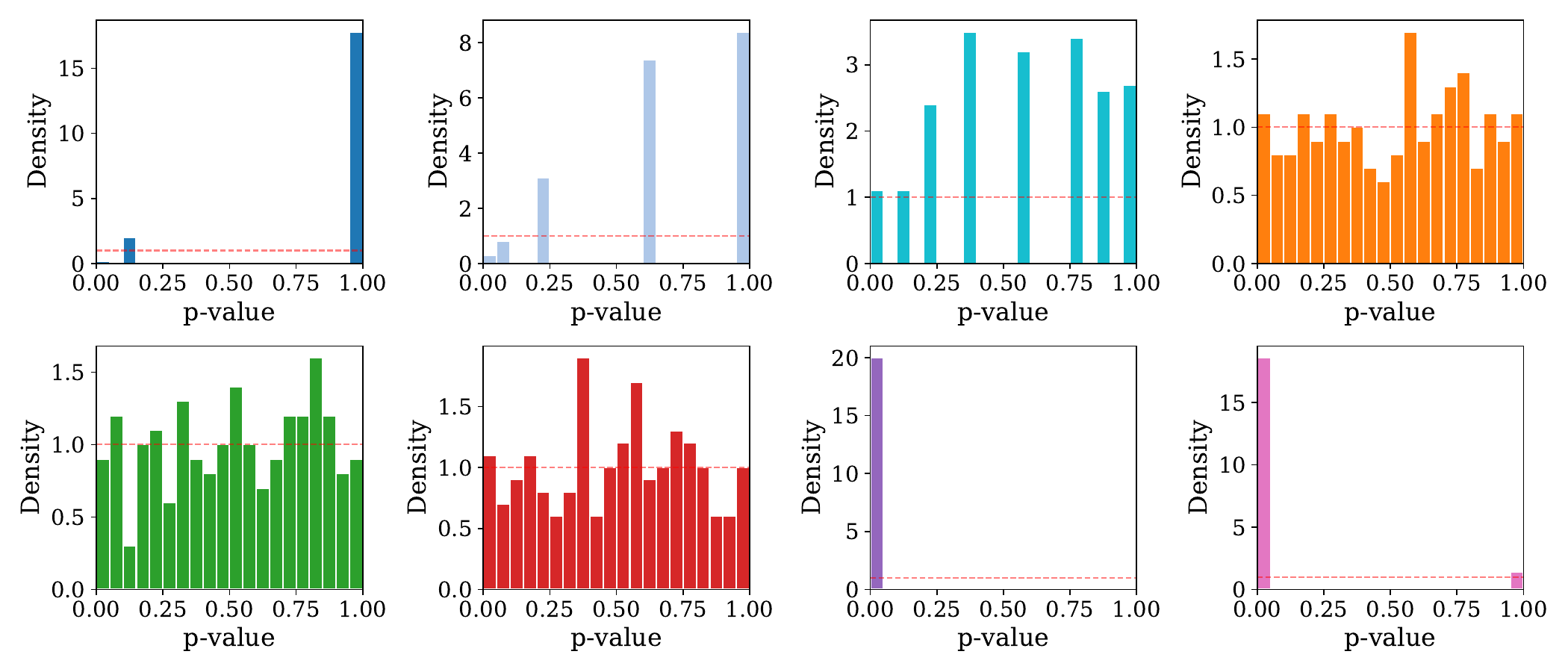}
        \caption{Conditional scores.}
        \label{fig:histogram_3_conditional}
    \end{subfigure}
    \hfill
    \begin{subfigure}[b]{\textwidth}
        \centering
        \includegraphics[width=.8\textwidth]{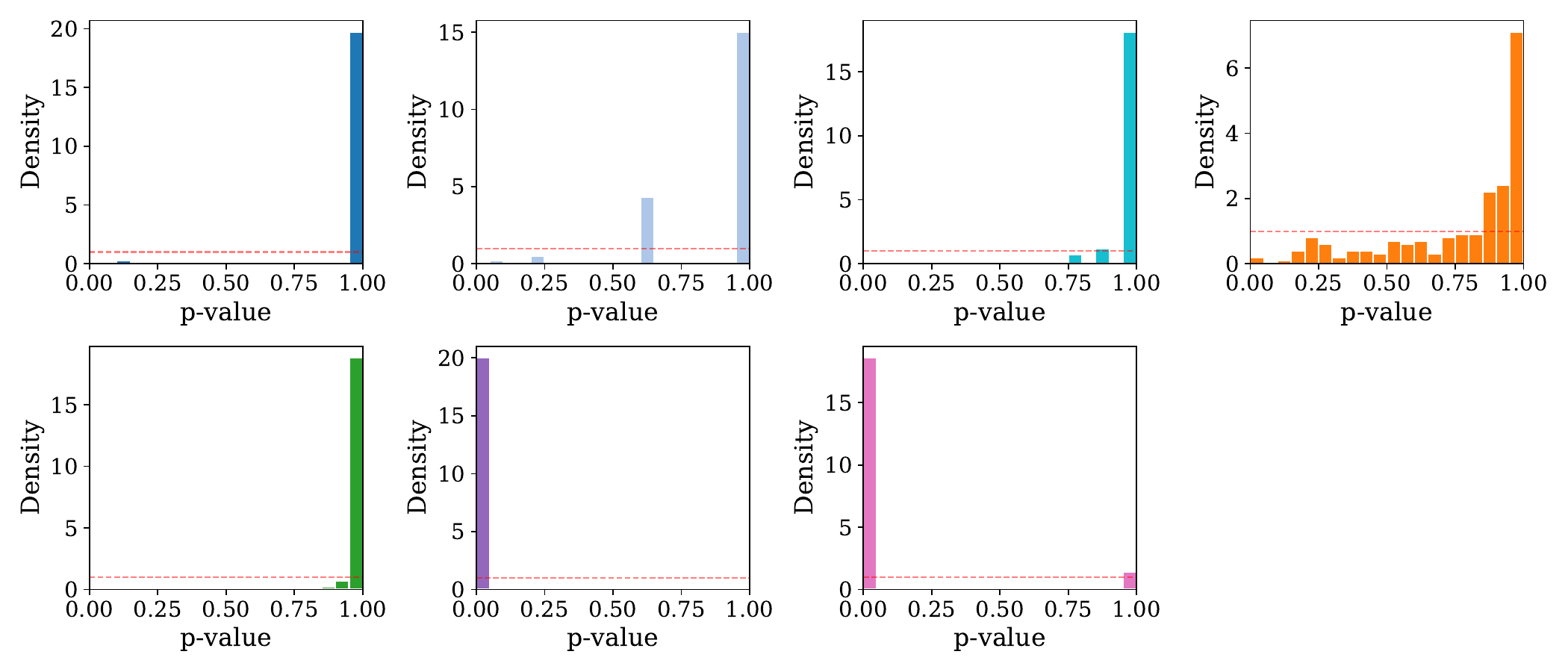}
        \caption{Marginal scores.}
        \label{fig:histogram_3_marginal}
    \end{subfigure}
    \caption{Histograms of the p-values of our tests under the null hypothesis.}
    \label{fig:histogram_3}
\end{figure*}
    
\subsection{Stronger Anomalies}
In \cref{fig:power_analysis_sigma_6} we repeated the experiment from \cref{fig:power_analysis} with the only difference being that we inject stronger anomalies. 
In particular, we increase or decrease the noise of the root cause by six instead of three standard deviations.
We can see that this indeed increases the power of our test.

\begin{figure*}[t]
\includegraphics[width=.7\textwidth]{img/linear_3/legend_perturbed_dag.pdf}
    \centering
    \begin{subfigure}[b]{0.45\textwidth}
        \centering
        \includegraphics[width=\textwidth]{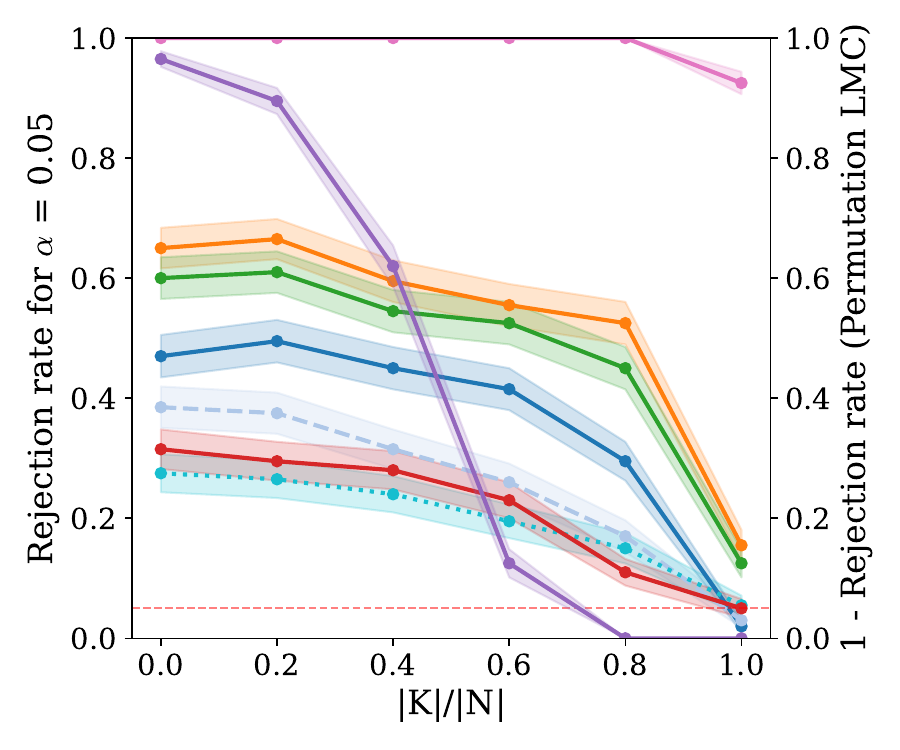}
        \caption{Conditional scores.}
        \label{fig:power_analysis_6_conditional}
    \end{subfigure}
    \hfill
    \begin{subfigure}[b]{0.45\textwidth}
        \centering
        \includegraphics[width=\textwidth]{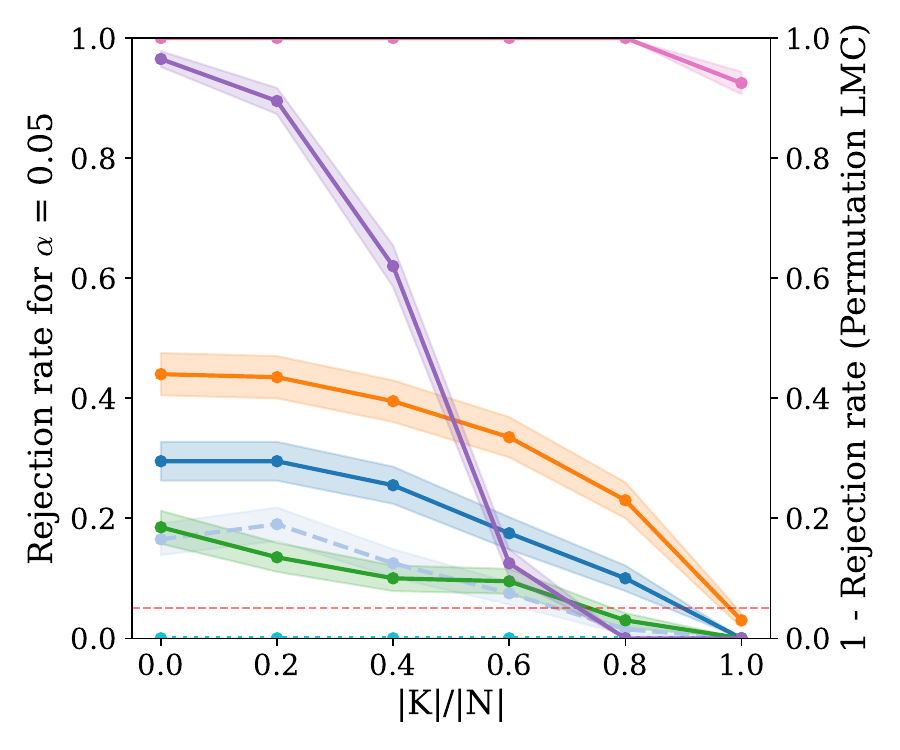}
        \caption{Marginal scores.}
        \label{fig:power_analysis_6_marginal}
    \end{subfigure}
    \caption{Average rejection rate for synthetic data with larger anomalies and increasing fraction of correctly connected nodes $|K|/|N|$. We observe a similar behaviour as with the smaller anomalies, except that our tests seem to have even more power.}
    \label{fig:power_analysis_sigma_6}
\end{figure*}

\subsection{More anomalies}
It is a direct consequence of the proofs of \cref{prop:binomialtest,prop:fishertest,prop:kstest,prop:tippetttest,prop:marginaltests} that we can also conduct the tests with more than a single sample, even if these anomalies correspond to different root cause events.
In order to do this, we simply need to drop all root cause scores (or score gaps) and combine the p-values of the remaining nodes for all samples.
Specifically, for \cref{prop:fishertest} this means summing over all of them to get $S_{\text{sum}}^\cG$.
For \cref{prop:tippetttest} this means taking the max over all of them for $S_{\text{max}}^\cG$.
For \cref{prop:binomialtest} we threshold and sum over all nodes and for \cref{prop:kstest} we use all these scores to estimate the empirical distribution of the conditional scores.

In \cref{fig:power_analysis_more_anomalies} we repeated the experiments from \cref{fig:power_analysis,fig:power_analysis_sigma_6} with 100 anomalous samples (with identical root cause) per anomaly event.
As we can see, especially the tests based on conditional scores gain power this way.
The gains for the marginal tests are not as pronounced, potentially because the multiple-testing burden overwhelms the weak signal in the data in many cases.
Interestingly, we can see that in this scenario the test based on Fisher's method and marginal scores seems to outperform Tippett's version, which is the opposite behaviour of the experiments in \cref{fig:power_analysis,fig:power_analysis_sigma_6}.
This is in line with our theory, as Tippett is expected to have optimal power if there is a high outlier score and Fisher when many of the combined scores are elevated.

\begin{figure*}[t]
    \centering
    \includegraphics[width=.7\textwidth]{img/linear_3/legend_perturbed_dag.pdf}
        \begin{subfigure}[b]{0.45\textwidth}
        \centering
        \includegraphics[width=\textwidth]{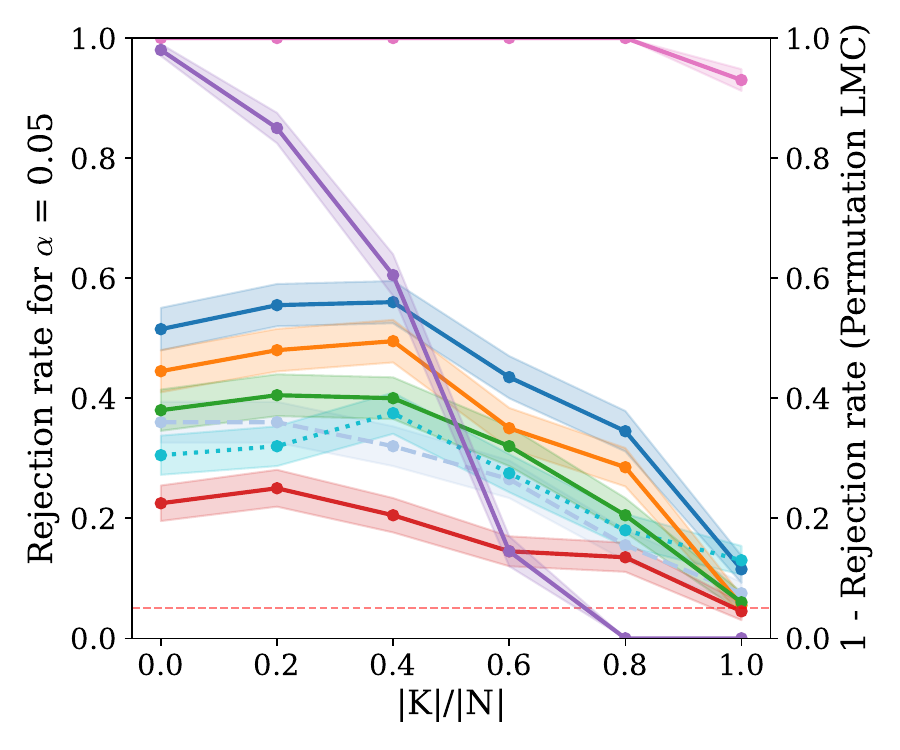}
        \caption{Conditional scores. Small anomalies.}
        \label{fig:power_analysis_3_conditional_more_anomalies}
    \end{subfigure}
    \hfill
    \begin{subfigure}[b]{0.45\textwidth}
        \centering
        \includegraphics[width=\textwidth]{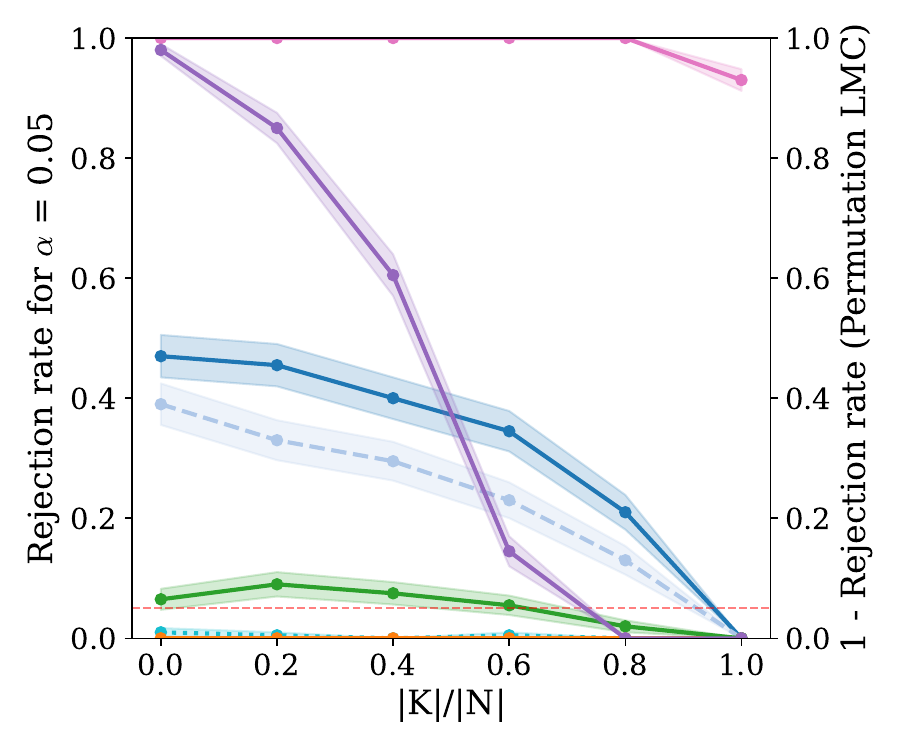}
        \caption{Marginal scores. Small anomalies.}
        \label{fig:power_analysis_3_marginal_more_anomalies}
    \end{subfigure}
    
    \begin{subfigure}[b]{0.45\textwidth}
        \centering
        \includegraphics[width=\textwidth]{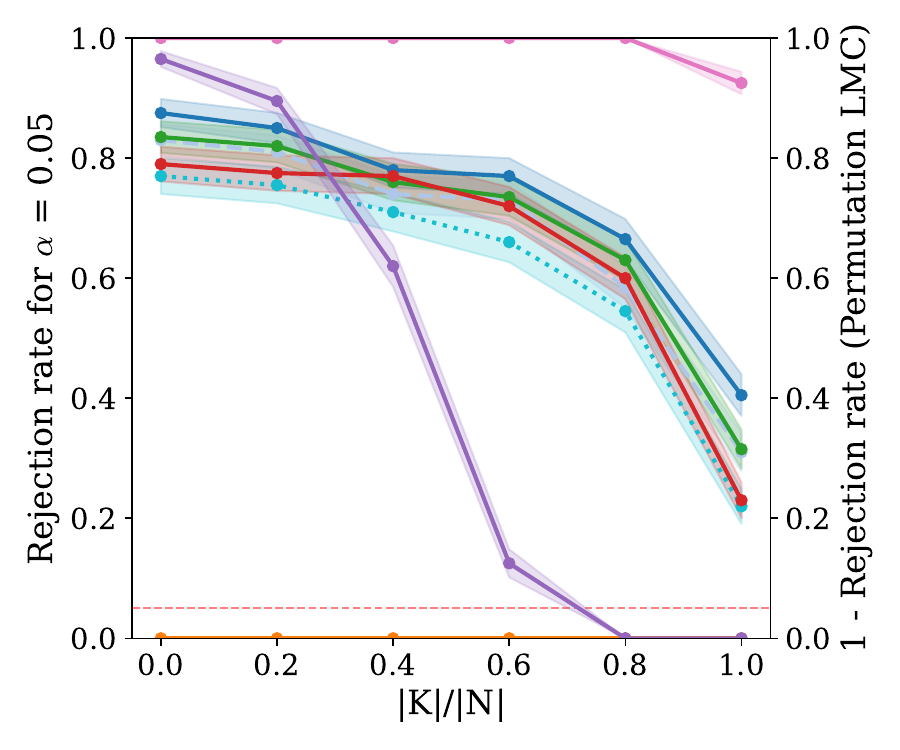}
        \caption{Conditional scores. Large anomalies.}
        \label{fig:power_analysis_6_conditional_more_anomalies}
    \end{subfigure}
    \hfill
    \begin{subfigure}[b]{0.45\textwidth}
        \centering
        \includegraphics[width=\textwidth]{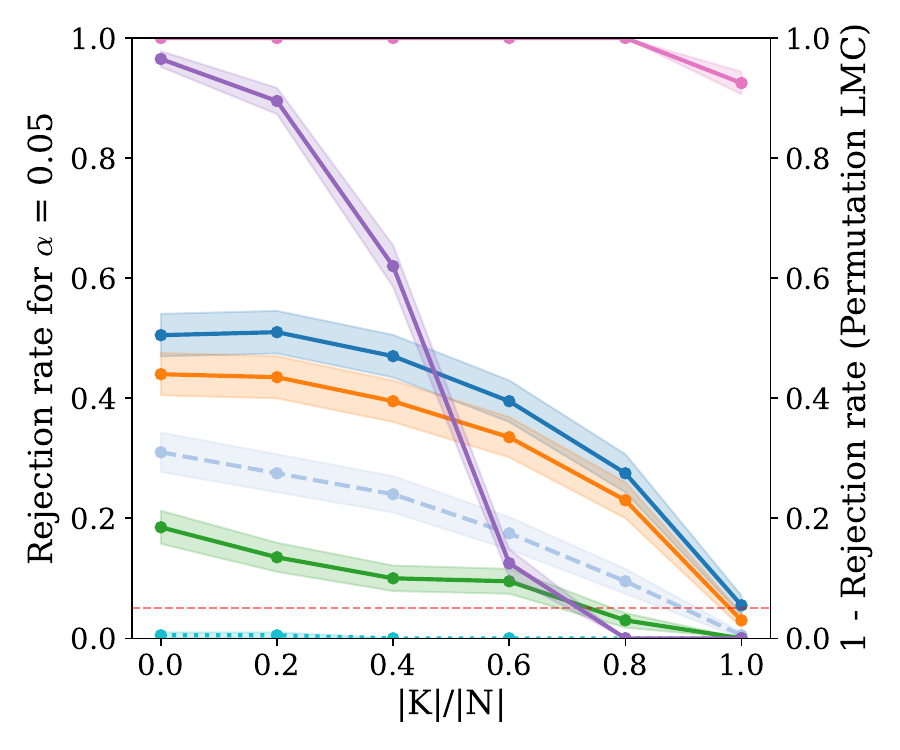}
        \caption{Marginal scores. Large anomalies.}
        \label{fig:power_analysis_6_marginal_more_anomalies}
    \end{subfigure}
    \caption{Average rejection rate for synthetic data with 100 anomalous samples and increasing fraction of correctly connected nodes $|K|/|N|$. We observe similar trends as before just with even more power.}
    \label{fig:power_analysis_more_anomalies}
\end{figure*}

\subsection{No Given Root Cause}
In \cref{fig:power_analysis_no_rc} we repeated the experiments from \cref{fig:power_analysis,fig:power_analysis_sigma_6} without assuming to know the root cause.
We can see that the proposed approach loses very little power in our experiments.

\begin{figure*}[t]
\includegraphics[width=.7\textwidth]{img/linear_3/legend_perturbed_dag.pdf}
    \centering
        \begin{subfigure}[b]{0.45\textwidth}
        \centering
        \includegraphics[width=\textwidth]{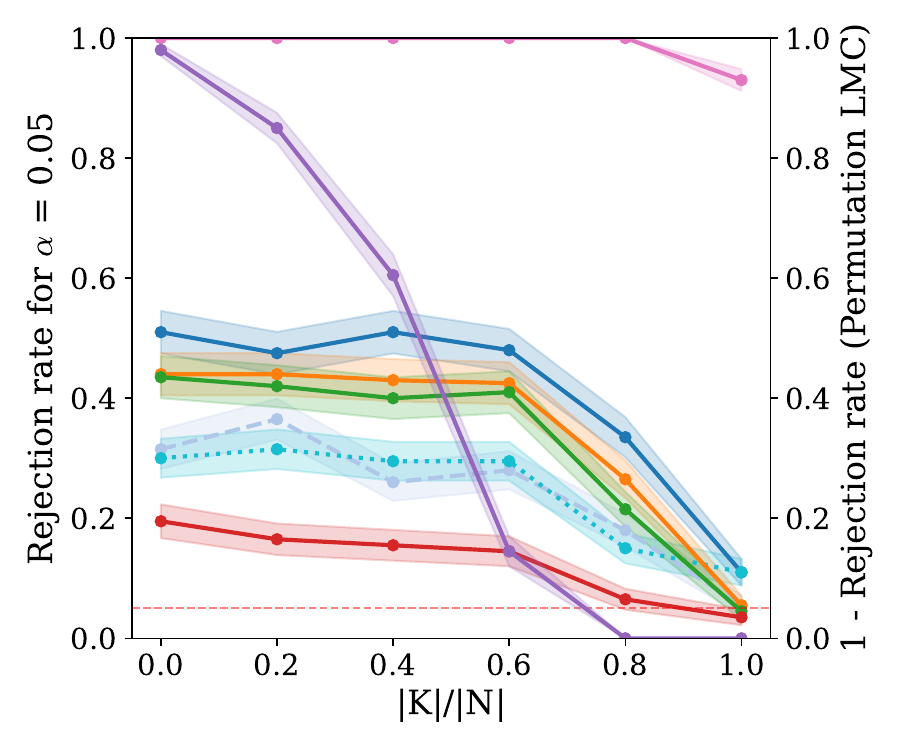}
        \caption{Conditional scores. Small anomalies.}
        \label{fig:power_analysis_3_conditional_no_rc}
    \end{subfigure}
    \hfill
    \begin{subfigure}[b]{0.45\textwidth}
        \centering
        \includegraphics[width=\textwidth]{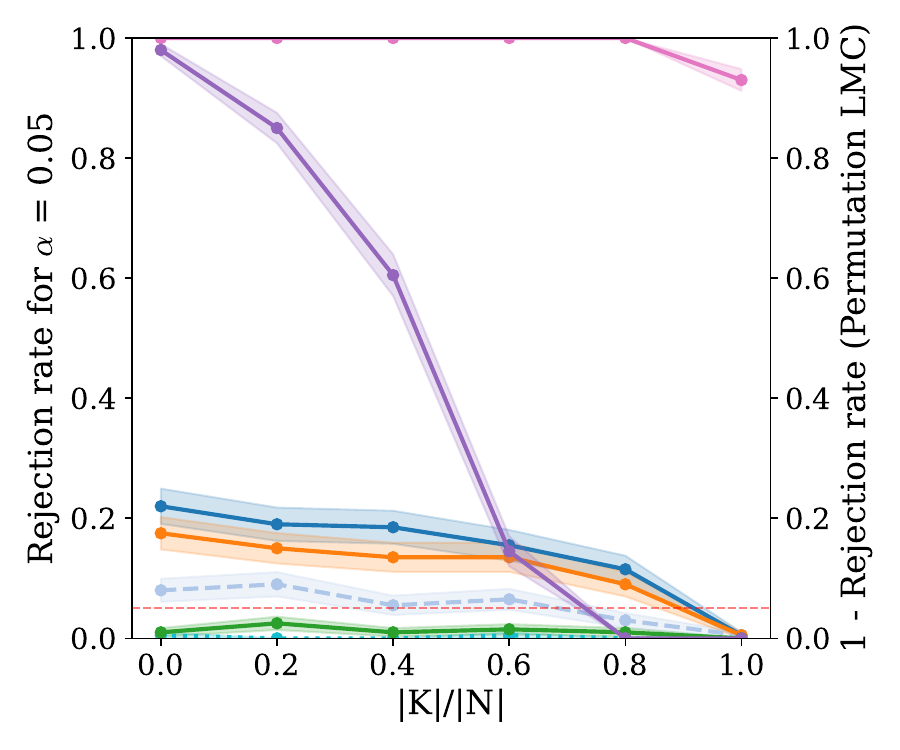}
        \caption{Marginal scores. Small anomalies.}
        \label{fig:power_analysis_3_marginal_no_rc}
    \end{subfigure}
    
    \begin{subfigure}[b]{0.45\textwidth}
        \centering
        \includegraphics[width=\textwidth]{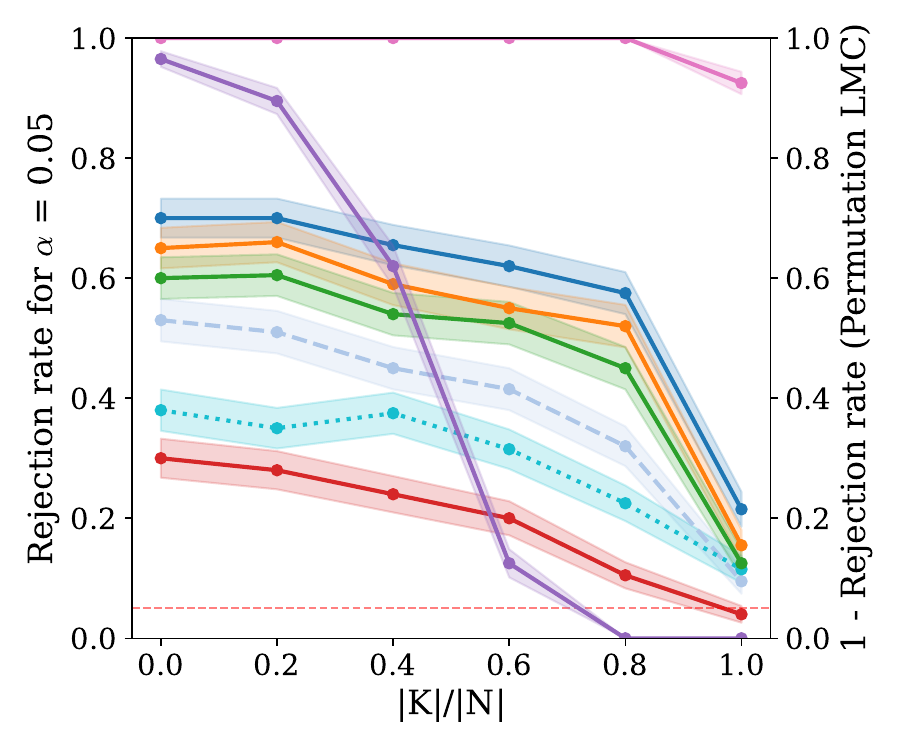}
        \caption{Conditional scores. Large anomalies.}
        \label{fig:power_analysis_6_conditional_no_rc}
    \end{subfigure}
    \hfill
    \begin{subfigure}[b]{0.45\textwidth}
        \centering
        \includegraphics[width=\textwidth]{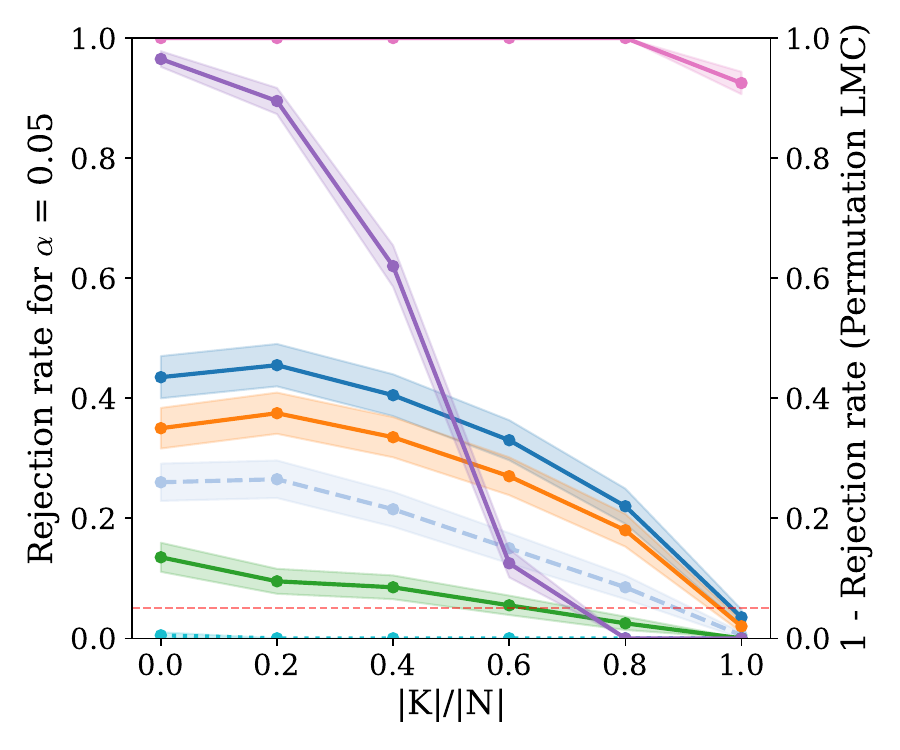}
        \caption{Marginal scores. Large anomalies.}
        \label{fig:power_analysis_6_marginal_no_rc}
    \end{subfigure}
    \caption{Average rejection rate for synthetic data when no root cause is given and increasing fraction of correctly connected nodes $|K|/|N|$. We observe similar trends as before just slightly less power.}
    \label{fig:power_analysis_no_rc}
\end{figure*}

\subsection{Graphs from Causal Discovery}
We further wanted to investigate the performance of our tests on graphs that have been estimated using a causal discovery method.
We generate the data analogously to before, except that we draw all noise terms uniformly from $[-1, 1)$.
This renders the data amenable for the DirectLinGAM algorithm \citep{shimizu2011directlingam}.
We estimate 200 graphs on samples of size 50, 100, 200, 400, 800 and 1600, respectively.
In \cref{fig:power_analysis_lingam} we sort each resulting graph in one of five bins of equal size with respect to the structural Hamming distance to the ground truth graph, i.e. the number of edges that would need to change in order to match the ground truth graph.
As we can see, the rejection rate of our tests increases as SHD gets larger.
In contrast to previous experiments, the permutation test baseline rejects most graphs (i.e. considers them \emph{better} than random).
This could be due to the fact that the graphs in this experiment are learned from the given data and therefore indeed represent the conditional independence structure of the data better than a random graph.

\begin{figure*}[t]
    \centering
    \includegraphics[width=.7\textwidth]{img/linear_3/legend_perturbed_dag.pdf}
    \begin{subfigure}[b]{0.45\textwidth}
        \centering
        \includegraphics[width=\textwidth]{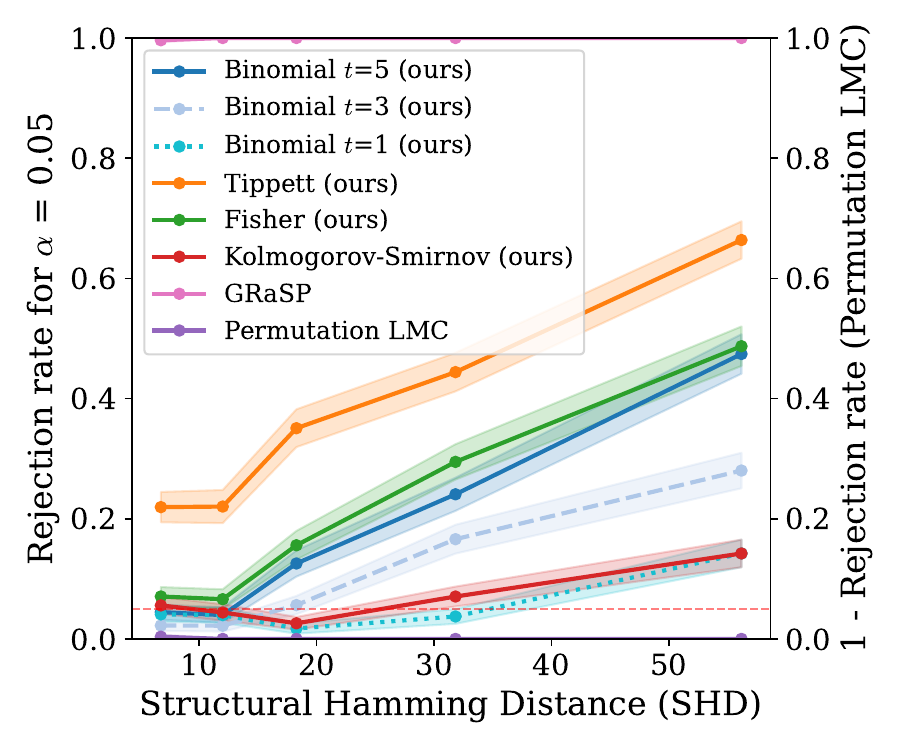}
        \caption{Conditional scores.}
        \label{fig:power_analysis_lingam_conditional}
    \end{subfigure}
    \hfill
    \begin{subfigure}[b]{0.45\textwidth}
        \centering
        \includegraphics[width=\textwidth]{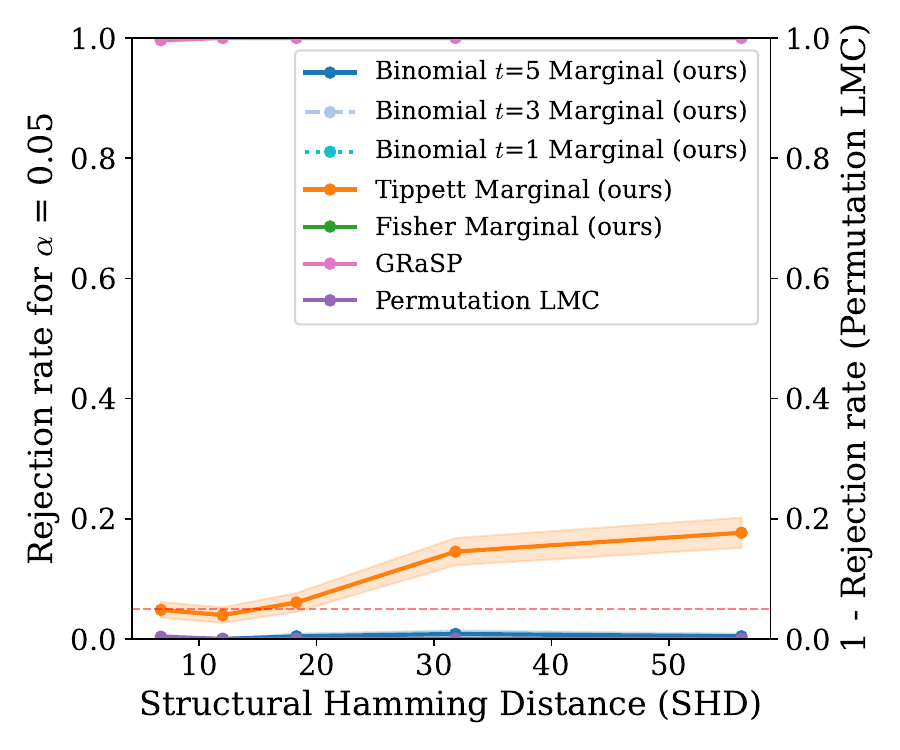}
        \caption{Marginal scores.}
        \label{fig:power_analysis_lingam_marginal}
    \end{subfigure}
    \caption{Average rejection rate for synthetic data with linear non-Gaussian data and graphs estimated using DirectLiNGAM \citep{shimizu2011directlingam}.}
    \label{fig:power_analysis_lingam}
\end{figure*}

\subsection{Non-Linear Mechanisms}
In \cref{fig:power_analysis_6_nonlinear} we repeated the experiment from \cref{fig:power_analysis} with the main difference that we generate the synthetic data from a model with non-linear causal mechanisms.
Precisely, we generate variable $X_i$ with non-empty parent set via polynomials of degree two in their parents, i.e.
\begin{displaymath}
    x_i = \sum_{X_j, X_k\in \PA_i^G} \alpha_{j,k}x_jx_k + n_i, 
\end{displaymath}
where $n_i$ is an independent standard Gaussian noise term like before and $\alpha_{i, j}\in \R$.
We initially draw $\alpha_{i,j}$ uniformly from $[-1, 1]$ and generate an (unnormalised) sample from these coefficients and the normal-regime samples from the parents. 
We then divide the coefficients by half the empirical standard deviation of this initial sample and re-generate the sample from the same parent values.
This way we get normal data with stable variances but in the anomalous regime outliers can still propagate.

For the permutation test baseline we used the generalised covariance measure \citep{shah2020hardness} as implemented in \texttt{DoWhy GCM} and we use gradient boosting regression as implemented in \texttt{scikit-learn} to estimate the conditionals.

For computational reasons we only consider 20 different DAGs in this experiment.

As we can see the tests behave similarly to our previous experiments.

\begin{figure*}[t]
    \centering
    \includegraphics[width=.7\textwidth]{img/linear_3/legend_perturbed_dag.pdf}

    \begin{subfigure}[b]{0.45\textwidth}
        \centering
        \includegraphics[width=\textwidth]{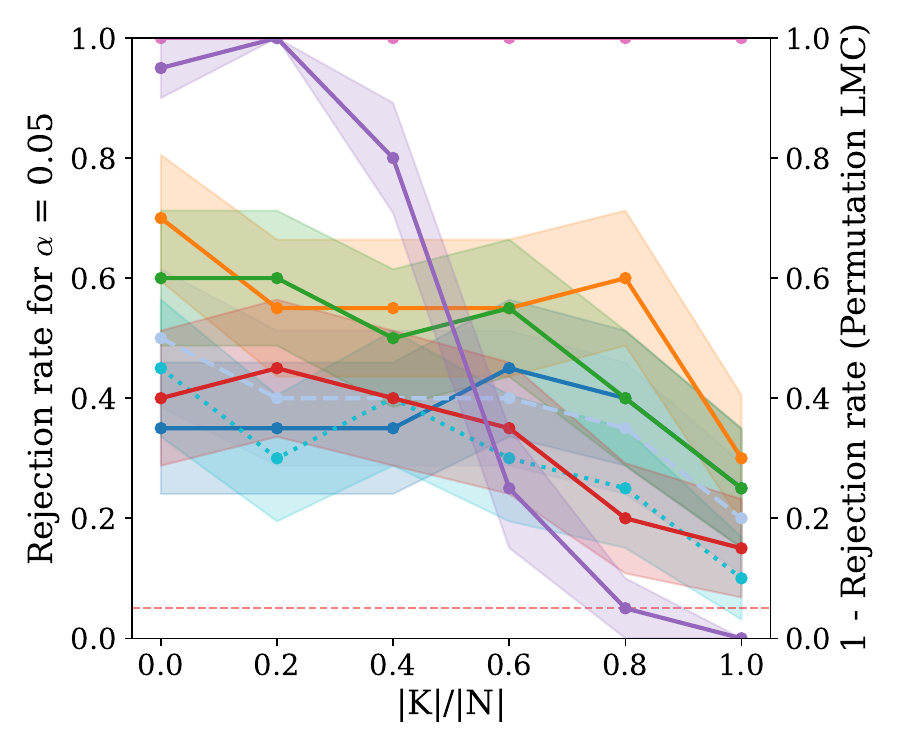}
        \caption{Conditional scores.}
        \label{fig:power_analysis_6_nonlinear_conditional}
    \end{subfigure}
    \hfill
    \begin{subfigure}[b]{0.45\textwidth}
        \centering
        \includegraphics[width=\textwidth]{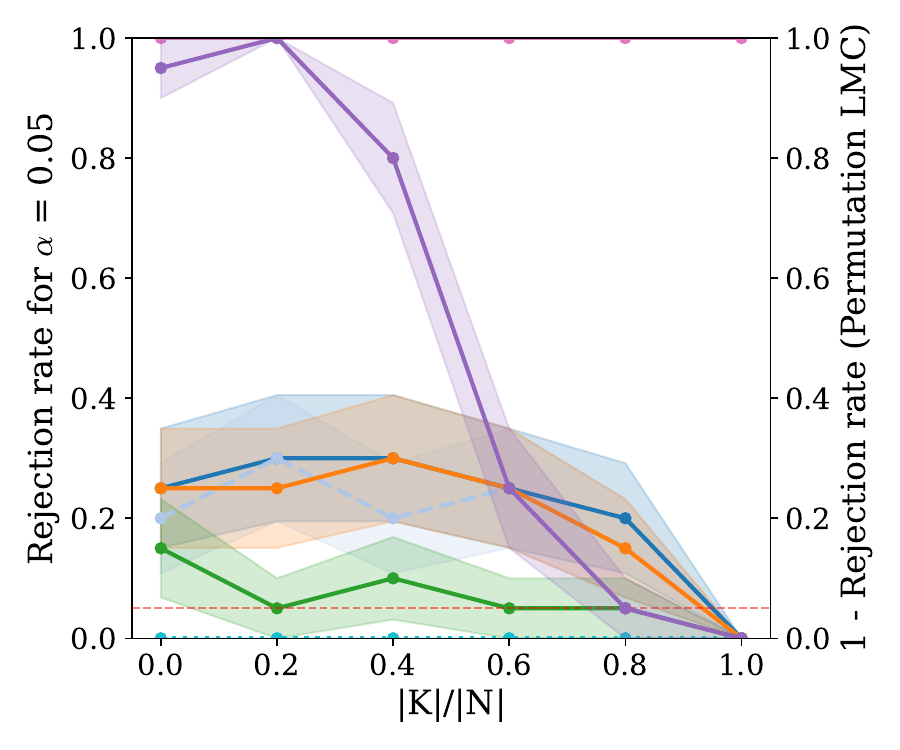}
        \caption{Marginal scores.}
        \label{fig:power_analysis_6_nonlinear_marginal}
    \end{subfigure}
    \caption{Average rejection rate for synthetic data with non-linear causal mechanisms and increasing fraction of correctly connected nodes $|K|/|N|$.}
    \label{fig:power_analysis_6_nonlinear}
\end{figure*}

\subsection{Interventional Samples for Permutation Baseline}
In \cref{fig:power_analysis_3_f} we repeated the experiment from \cref{fig:power_analysis} and added a so-called $F$-node \citep{Pearl:00} to the graph used by the permutation test baseline.
More precisely, we concatenate each normal dataset with the outlier sample and add a column $F$ that indicates whether the given row is from the normal or outlier regime.
We also add a node $F$ to the estimated causal graph that has a single outgoing edge pointing to the known root cause.
We then proceed as before.
As expected, the single outlier sample is not enough to substantially change the behaviour of the test.

\begin{figure*}[t]
    \centering
    \includegraphics[width=.7\textwidth]{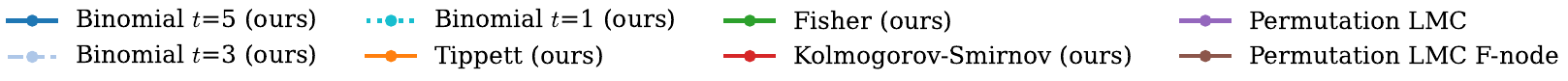}

    \begin{subfigure}[b]{0.45\textwidth}
        \centering
        \includegraphics[width=\textwidth]{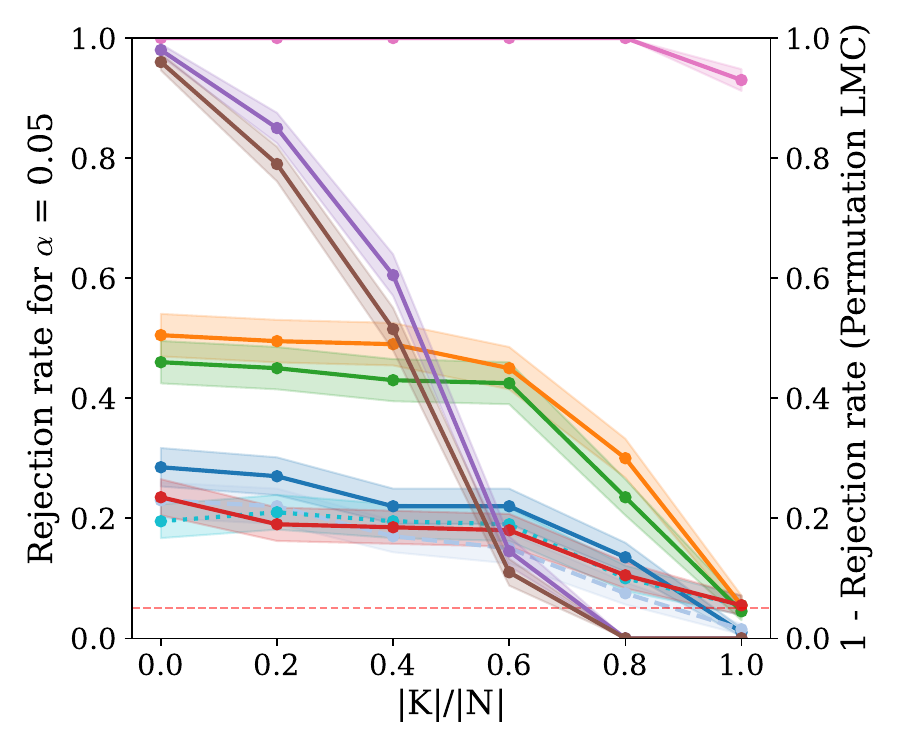}
        \caption{Conditional scores.}
        \label{fig:power_analysis_3_f_conditional}
    \end{subfigure}
    \hfill
    \begin{subfigure}[b]{0.45\textwidth}
        \centering
        \includegraphics[width=\textwidth]{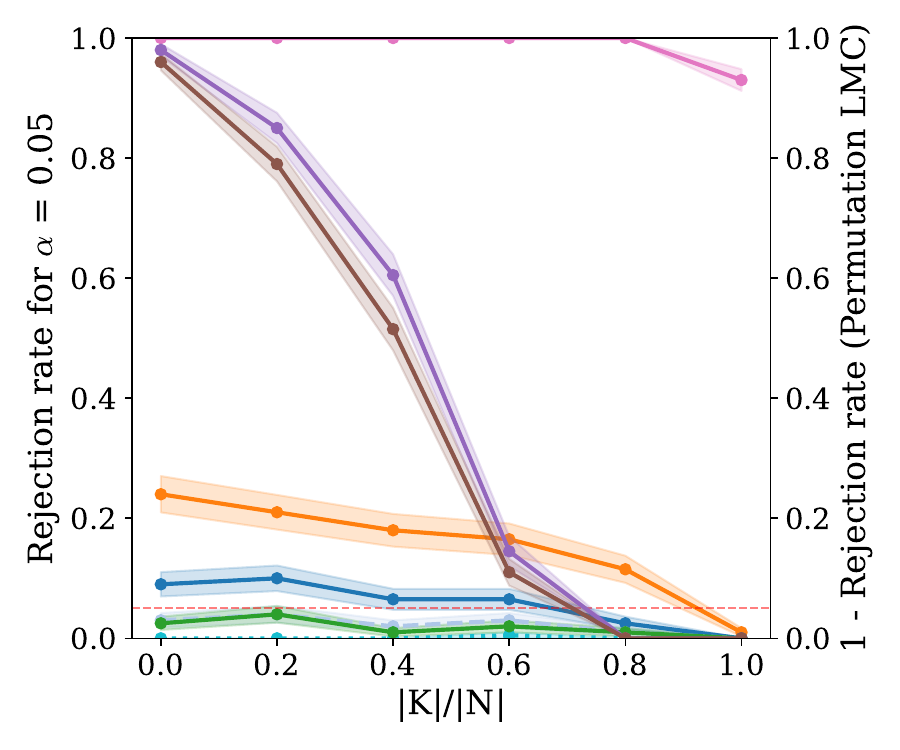}
        \caption{Marginal scores.}
        \label{fig:power_analysis_3_f_marginal}
    \end{subfigure}
    \caption{Average rejection rate for synthetic data with an increasing fraction of correctly connected nodes $|K|/|N|$. We added an $F$-node to the permutation test baseline.}
    \label{fig:power_analysis_3_f}
\end{figure*}

\subsection{Conditional tests on PetShop and Causal Chambers data}
In \cref{fig:petshop_conditional} we repeated the experiments in \cref{fig:petshop_marginal} but with the conditional version of the tests from \cref{prop:fishertest,prop:binomialtest,prop:kstest,prop:tippetttest}.
We can see that the tests reject even more often than their marginal counterparts.
In contrast to their marginal counterparts, we can see that also the Causal Chambers graph gets rejected slightly more often than we would expect from the selected significance threshold. 
Comparing with the synthetic data in \cref{fig:power_analysis_3_conditional,fig:power_analysis_6_conditional}, this might be due to statistical errors in the estimation of the conditional scores.

\begin{figure}[t]
    \centering
    \includegraphics[width=.7\textwidth]{img/linear_3/legend_perturbed_dag.pdf}
    \includegraphics[width=.45\linewidth]{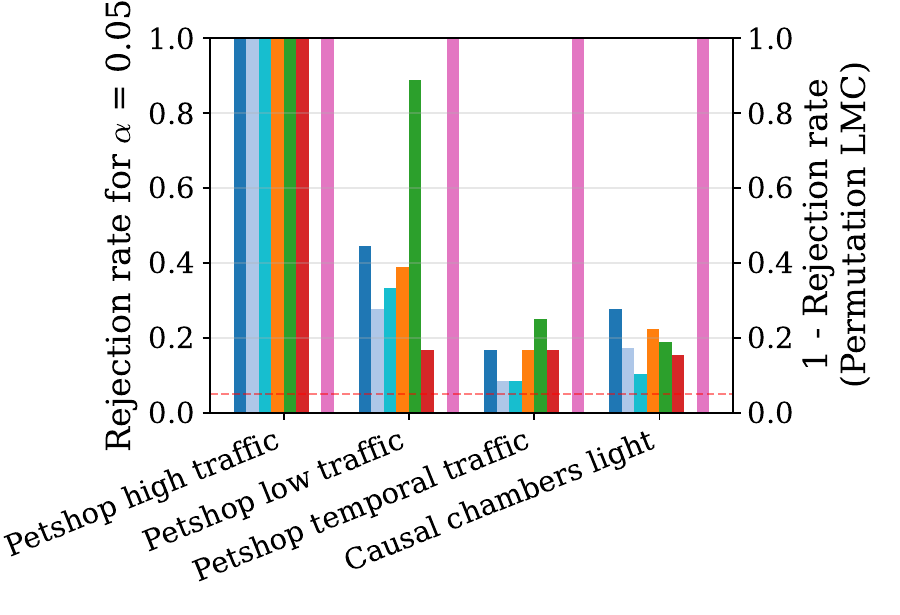}
    \caption{Average rejection rate of our conditional tests on cloud computing and causal chambers data. While the inverted call graph seems to represent the cloud data better than random, our test rejects this graph as ground truth for several anomaly events. The graph from causal chambers gets rejected more rarely but still slightly above the prescribed false positive rate.}
    \label{fig:petshop_conditional}
\end{figure}

\end{document}